\documentclass{article}
\pdfoutput=1

\usepackage{amsmath,amsfonts,bm}

\def\Figref#1{Figure~\ref{#1}}

\def\eqref#1{equation~\ref{#1}}

\def\1{\bm{1}}

\def\rc{{\textnormal{c}}}

\def\rg{{\textnormal{g}}}

\def\vc{{\bm{c}}}

\def\vr{{\bm{r}}}

\def\vx{{\bm{x}}}

\def\vz{{\bm{z}}}

\def\mM{{\bm{M}}}

\DeclareMathAlphabet{\mathsfit}{\encodingdefault}{\sfdefault}{m}{sl}
\SetMathAlphabet{\mathsfit}{bold}{\encodingdefault}{\sfdefault}{bx}{n}

\def\gG{{\mathcal{G}}}

\def\sC{{\mathbb{C}}}

\def\sG{{\mathbb{G}}}

\def\sX{{\mathbb{X}}}

\def\sZ{{\mathbb{Z}}}

\newcommand{\E}{\mathbb{E}}

\newcommand{\R}{\mathbb{R}}

\newcommand{\JSD}{D_{\mathrm{JS}}}

\usepackage{xcolor}

\usepackage{microtype}
\usepackage{graphicx}
\usepackage{booktabs} 
\usepackage{caption}
\usepackage{subcaption}
\usepackage{float}
\usepackage{placeins}
\usepackage{tabularx}
\usepackage{multirow}
\usepackage{rotating}
\usepackage{xcolor}

\usepackage{hyperref}

\usepackage[accepted]{icml2023}

\usepackage{amsmath}
\usepackage{amssymb}
\usepackage{mathtools}
\usepackage{amsthm}

\usepackage[capitalize,noabbrev]{cleveref}
\usepackage[english]{babel}
\usepackage{xspace}
\usepackage{svg}

\theoremstyle{plain}
\newtheorem{theorem}{Theorem}[section]

\newtheorem{corollary}[theorem]{Corollary}
\theoremstyle{definition}

\theoremstyle{remark}

\usepackage[disable,textsize=tiny]{todonotes}

\newcommand{\method}{DUET\xspace}
\newcommand{\stovez}{$\text{\method}_{\lambda=0}$\xspace}

\newcommand{\ie}{\textit{i.e.}\xspace}
\newcommand{\eg}{\textit{e.g.,}\xspace}

\newcommand{\dg}{{G}}
\newcommand{\dc}{{C}}
\newcommand{\Normal}{\mathcal{N}}
\newcommand{\Loss}{L}

\newcommand{\VM}{\text{vM}}

\newcommand{\change}[1]{{#1}}

\newcommand{\ourtitle}{DUET: 2D Structured and \change{Approximately} Equivariant Representations}
\icmltitlerunning{\ourtitle}

\begin{document}

\twocolumn[
\icmltitle{\ourtitle}



\icmlsetsymbol{equal}{*}

\begin{icmlauthorlist}
\icmlauthor{Xavier Suau}{Apple}
\icmlauthor{Federico Danieli}{Apple}
\icmlauthor{T. Anderson Keller}{Apple,UvA}
\icmlauthor{Arno Blaas}{Apple}
\icmlauthor{Chen Huang}{Apple}
\icmlauthor{Jason Ramapuram}{Apple}
\icmlauthor{Dan Busbridge}{Apple}
\icmlauthor{Luca Zappella}{Apple}
\end{icmlauthorlist}

\icmlaffiliation{Apple}{Apple}
\icmlaffiliation{UvA}{University of Amsterdam}

\icmlcorrespondingauthor{Xavier Suau}{xsuaucuadros@apple.com}

\icmlkeywords{Equivariance, SSL, robustness, augmentation}

\vskip 0.3in
]



\printAffiliationsAndNotice{} 

\begin{abstract}
 \todo{rewrite abstract}
    Multiview Self-Supervised Learning (MSSL) is based on learning invariances with respect to a set of input transformations. However, invariance partially or totally removes transformation-related information from the representations, which might harm performance for specific downstream tasks that require such information. 
    We propose 2D strUctured and approximately EquivarianT representations (coined \method), which are 2d representations organized in a matrix structure, and equivariant with respect to transformations acting on the input data. \method representations maintain information about an input transformation, while remaining semantically expressive.
    Compared to SimCLR \cite{simclr} (unstructured and invariant) and ESSL \cite{essl} (unstructured and equivariant), the structured and equivariant nature of \method representations enables controlled generation with lower reconstruction error, while controllability is not possible with SimCLR or ESSL. \method also achieves higher accuracy for several discriminative tasks, and improves transfer learning.

\end{abstract}

\section{Introduction}





\begin{figure}[t]
\caption{\method. 
The backbone $f$ yields a 2d representation for each transformed image $f(\tau_g(\vx))$ (\eg $\tau_g$ is a rotation by $g$ degrees). 
The group marginal is obtained as the softmax (sm) of the sum of the rows, and is compared to the prescribed target (red) with our group loss $L_\sG$. 
The content is obtained by summing the columns, and contrasted ($L_\sC$) with the other view through a projection head $h$. 
The final representation for downstream tasks is the 2d one, which has been optimized through its marginals.}
\centering
\label{fig:schema}
\vskip 0.1in
\includegraphics[width=78mm,  height=52mm]{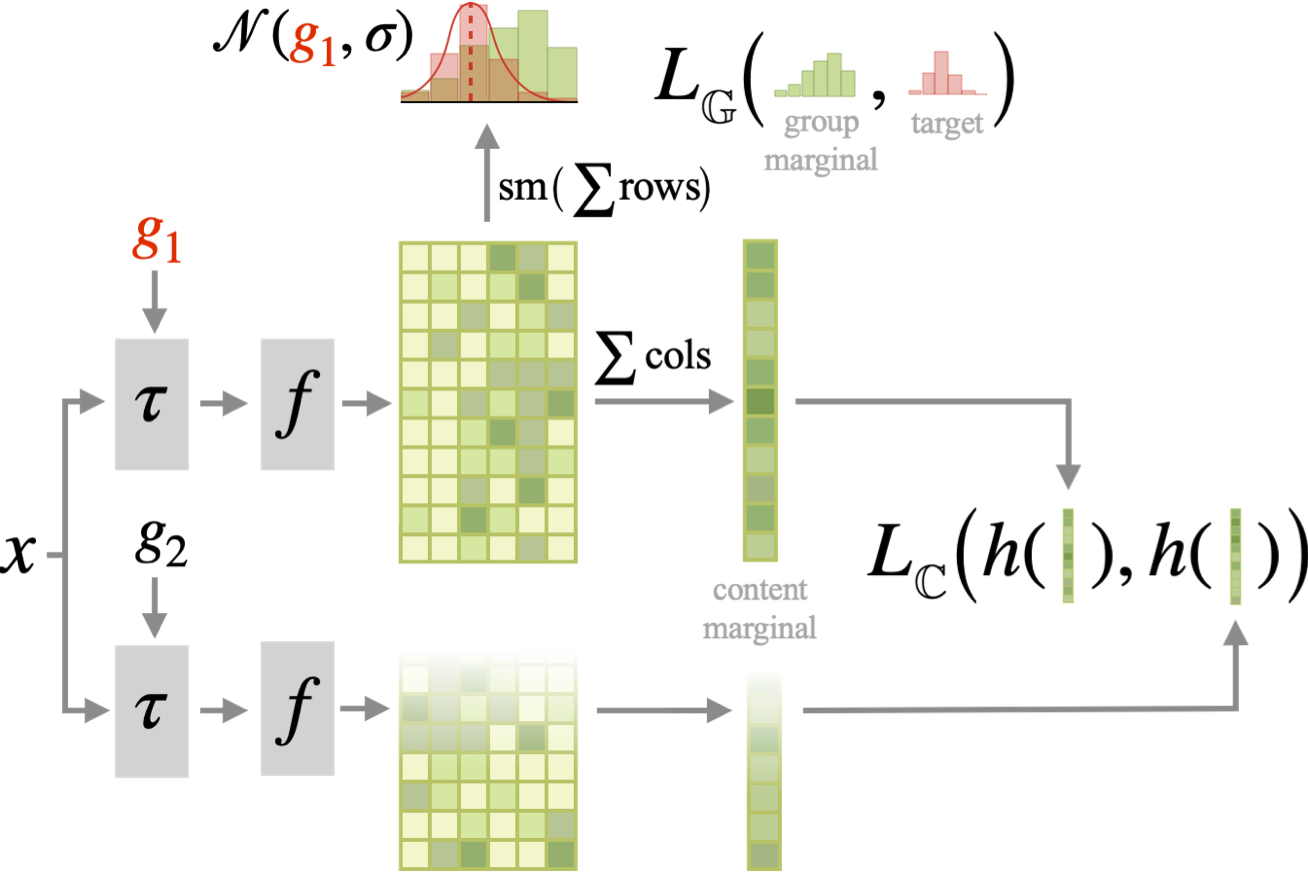}
\vskip -0.2in
\end{figure} 

The field of representation learning has evolved at a rapid pace in recent years, partially due to the popularity of Multiview Self-Supervised Learning (MSSL)~\cite{simclr, moco, swav, byol, btwins}. 
The main idea of MSSL is to learn transformation-invariant representations by comparing data views that underwent different transformations. 
If the transformations alter \emph{only} task-irrelevant information, and if representations of multiple views are similar, then those representations should only contain task-relevant information. 
However, one can always find a downstream task for which the chosen transformations are relevant. 
For example, MSSL representations which learn to be color invariant are likely to fail at predicting fruit ripeness where color information is required \cite{makes_good_views}, or at the tasks of generation or segmentatation \cite{hybgencon}.

One way to maintain information in the representations is by preserving all possible information from the input, as pursued by InfoMax \cite{InfoMax} frameworks. 
However, it has been shown empirically and theoretically that for tasks like classification, invariance to nuisance information allows for greater data efficiency and downstream performance \cite{tipooling, MI_Max}. 
In an attempt to simultaneously satisfy the demands of information-rich representations (allowing for generalization to different tasks) and complex invariances (allowing for powerful discriminative representations), modern machine learning research has pursued the concept of structured representations.
Colloquially, a representation can be considered structured with respect to a set of transformations if firstly, the transformation between two inputs can be easily recovered by comparing their representations, and secondly, there is a known method for recovering the representation which is correspondingly invariant to the transformation set. 
An example of structured representations are convolutional feature maps, which allow for the spatial position (the translation element) to be easily extracted, while similarly allowing for global translation invariance through spatial pooling. 
Given the success of structured representations, significant work has gone into expanding the range of transformations for which a structured representation can be recovered, for example, rotation, scaling, and other algabraic group symmetries \cite{gcnn, sesn, arblie, qtae, offseteq}.

In the context of MSSL, equivariance has also been successfully used to improve distributional robustness \cite{essl, augself, hssl}.
However, to date, this equivariance has largely been encouraged at an informational level rather than a structural level, making the careful disassociation of the equivariant and invariant aspects of the representation challenging or impossible.
For example, ESSL \cite{essl} and AugSelf \cite{augself} make representations sensitive to a transformation by regressing the transformation parameter, making their representations theoretically equivariant, but not interpretably structured, as there is no explicit form of the transformation at representation level. Such a lack of structure makes computing invariances or controlled generation from such representations significantly more challenging.

In this work, we present \method, a method to learn structured \emph{and} equivariant representations with MSSL.
Instead of learning 1-dimensional representations as in SimCLR \cite{simclr} or ESSL, \method representations are reshaped to 2d (see \Cref{fig:schema}). This allows for
a richer optimization through their row- and column-wise marginals, which are respectively related to the \textit{group element} (the transformation parameter, \eg rotation angle) and \textit{content} (all the information that is invariant to the transformation actions).
In summary, our main contributions are \todo{Update}:

\begin{itemize}
    \item We introduce \method, a method to incorporate interpretable structure in MSSL representations for both finite and infinite groups with negligible computational overhead.\footnote{Code available at \href{https://github.com/apple/ml-duet}{https://github.com/apple/ml-duet}}  Our approach also performs well for parameterized transformations that do not satisfy all algebraic group axioms \cite{serre}, making it widely applicable to most transformations used in MSSL.
    \item We show empirically that \method representations become approximately equivariant as a by-product of their predictiveness of a transformation parameter. Importantly, we prescribe an explicit form of transformation at representation level that enables controllable generation, not achievable with ESSL or SimCLR.
    \item We shed some light on why certain symmetries (\eg horizontal flips, color transformations) are harder to learn from typical computer vision datasets, due to inherent ambiguity in the data with respect to a transformation. For example, cars appear in both left and right directions, hence making it difficult to define what a \textit{non-flipped} car is. 
    \item We provide extensive experiments on several datasets, comparing with SimCLR and ESSL. We show that \method representations are suitable for discriminative tasks, transfer learning and controllable generation.
\end{itemize}


\section{Related Work}


\textbf{Structured and Equivariant Representations.} 
In the unsupervised learning domain, existing works like \cite{isa_vae} have extensively explored structured latent priors for the Variational Autoencoder (VAE) \cite{Kingma2014}, while the recent Topographic VAE \cite{tvae} aims to induce topographic organization of the observed set of transformations. The idea of structured representations has also been connected to unsupervised learning of disentangled representations \cite{betavae,kumar2018variational}. Another closely related line of work focuses on learning equivariance \cite{gcnn, sesn, arblie, qtae, offseteq} as a more general form of structured representations. For example, \cite{sesn} propose to use a basis of transformed filters to learn equivariant features, which generally leads to improved model robustness and data efficiency. NPTN \cite{nptn} follows on \cite{sesn} and proposes to use a completely learnt basis of filters, learning unsupervised invariances.

\textbf{Structure in MSSL.}
Modern MSSL is based on discarding task-irrelevant information via image augmentations. 
Contrastive and non-contrastive approaches achieve this respectively by comparing augmented views of different data \cite{simclr, moco, swav}, or by only comparing views from the same datum \cite{byol, btwins}.
Several authors have explored the comparison of spatially structured representations \cite{admim} (exploiting the InfoMax principle \cite{InfoMax}) or using variants of the NCE \cite{nce} loss \cite{gim, cpc, deepinfomax}.
Some works have studied the impact objectives have on the distributions of representations \citep{DBLP:conf/icml/0001I20}, and how these representations may be identifiable with the latent factors of the data generative process \citep{DBLP:conf/icml/ZimmermannSSBB21}.
Recent works have tackled the preservation of information in MSSL representations. ESSL \cite{essl} supplements SimCLR \cite{simclr} by predicting the parameter of a transformation of choice, and obtaining theoretically equivariant representations. Similarly, although not focused on equivariance, AugSelf \cite{augself} predicts the difference in transformation parameters between two views. PCL \cite{pcl_feature_suppression} adds a reconstruction loss to preserve information about the input. Concurrent work \cite{mast} disentangles the feature space with masks learned via augmentations.

% ...
\section{Preliminary Considerations}
\label{sec:preliminary}

\addtolength{\tabcolsep}{-2pt}    
\begin{table}[tb]
    \caption{Transformations considered with their associated parameters. Column $g$ shows the corresponding parameters for the group-marginal definition in \method, and \textit{Target} shows the recommended target distribution. Note that flips are mapped to $\{\frac{1}{4},\frac{3}{4}\}$, turning them into cyclic groups.} 
    \label{tab:groups}
    \footnotesize
    \vskip 0.1in
    \begin{tabular}{lcllc}
    \toprule
    Transform. & Finite & Parameter & $g$ & Target \\
    \midrule
    \vspace{1mm}
     Rot. (4-fold)       & $\surd$  & $\{0,90,180,270\}$ & $\{\frac{1}{8},\frac{3}{8},\frac{5}{8},\frac{7}{8}\}$ & $\VM$ \\
     \vspace{1mm}
     Rot. ($360$)  &          & $[-180, 180]$ & $[0, 1]$ & $\VM$ \\
     \vspace{1mm}
     H. Flip             & $\surd$  & $\{0, 1\}$ & $\{\frac{1}{4},\frac{3}{4}\}$ & $\VM$ \\
     \vspace{1mm}
     V. Flip             & $\surd$  & $\{0, 1\}$ & $\{\frac{1}{4},\frac{3}{4}\}$ & $\VM$ \\
     \vspace{1mm}
     Grayscale             & $\surd$  & $\{0, 1\}$ & $\{0, 1\}$ & $\mathcal{N}$ \\
     \vspace{1mm}
     Brightness             & & $[0.6, 1.4]$ & $[0, 1]$ & $\mathcal{N}$ \\
     \vspace{1mm}
     Contrast             & & $[0.6, 1.4]$ & $[0, 1]$ & $\mathcal{N}$ \\
     \vspace{1mm}
     Saturation             & & $[0.6, 1.4]$ & $[0, 1]$ & $\mathcal{N}$ \\
     \vspace{1mm}
     Hue             & & $[-0.1, 0.1]$ & $[0, 1]$ & $\mathcal{N}$ \\
     RRC          & & $[0.2W, W]$ & $[0, 1]$ & $\mathcal{N}$ \\
    \bottomrule
\end{tabular}
\vskip -0.1in
\end{table}
\addtolength{\tabcolsep}{1pt}  

\subsection{Groups and Equivariance}
Let $f: \sX \mapsto \sZ$ be a mapping from data to representations. Such mapping is equivariant to the algebraic group $\gG = (\sG, \cdot)$ if there exists an input transformation $\tau: \sG \times \sX \mapsto \sX $ (noted $\tau_g(\vx)$) and a representation transformation $T: \sG \times \sZ \mapsto \sZ $ (noted $T_g(\vz)$) so that
\begin{align}
\label{eq:equivariance}
    T_g(f(\vx)) 
    &= f(\tau_g(\vx)), 
    & 
    \forall 
    & g \in \sG,
    &
    \forall 
    & \vx \in \sX.
\end{align}
If $\tau_g$ and $T_g$ form algebraic groups in the input and representation spaces respectively, then the mapping $f$ preserves the structure of the input group in the representation space (homomorphism). Recall that for $(\sG, \cdot)$ to form a group, the properties of \textit{closure}, \textit{associativity}, and existence of \textit{neutral} and \textit{inverse} elements must be satisfied \cite{serre}. Here we consider both finite and infinite groups. 

\subsection{On MSSL Input Transformations} 
\label{sec:mssl_transformations}

In MSSL, $\tau_g$ is defined by a parameterized transformation applied to the input. For example, rotation is parameterized by a real angle ($g \in \R$, infinite group). If $g \in [0, 2\pi]$ then it forms a cyclic group. One can also use discrete rotations which form a finite group where $g \in \{0^\circ, 90^\circ, 180^\circ, 270^\circ\}$. However, not all input transformations form a group. For example, a change in image contrast moves some pixel values out of range, thus clipping is applied which invalidates the \textit{associativity} property (\eg $\tau_{2.0}(\tau_{0.5}(\vx)) \neq \tau_{0.5}(\tau_{2.0}(\vx))$).
We also include RandomResizedCrop (RRC) in our study using the relative cropped width $W$ as a proxy for scale (assuming loss of information about location and aspect ratio).
All transformation parameters are mapped in $[0, 1]$ by min-max normalization as shown in Table~\ref{tab:groups}.
Although some transformations do not form a group (\eg RRC, color transformations), the concept of equivariance is often relaxed to embrace transformations which do not form groups. Note that this assumption does not invalidate our methodology for exact groups, and helps understand how our method is suitable for non-exact groups.



\section{\method Representations}
\label{sec:method}

In this section we describe how we can use \method to learn representations that are structured with respect to an algebraic group $\gG = (\sG, \cdot)$. 
 The overall \method architecture is shown in \Cref{fig:schema}. A training input image $\vx$ is transformed twice by sampling 2 group actions from the same group $g_1, g_2 \in \sG$ (\eg two angles of rotation). We obtain the transformed images $\vx_k = \tau_{g_k}(\vx)$ with $k=1, 2$. Let $f$ be a deep neural network backbone such that $\vz_k = f(\vx_k) \in \R^{\dc\times \dg}$, where $\dc$ and $\dg$ are the number of rows and columns in the representation, as shown in \Cref{fig:schema}.
This 2-dimensional representation $\vz_k$ models the joint (discretized and unnormalized) distribution $P(\rc, \rg | \vx_k)$ where  $\rc \in \R^C$ and $\rg \in \sG$ are two random variables defined in the content and group element domains. Our joint interpretation allows to  marginalize $P(\rc | \vx_k)$ by summing the columns of $\vz_k$, and $P(\rg | \vx_k)$ by summing the rows. 
\change{Rather than imposing a certain dependence (or independence) structure between $\rc$ and $\rg$, (conditioned on $\vx_k$), we only impose our objectives on the marginals $P(\rc | \vx_k)$ and $P(\rg | \vx_k)$ and let the model learn such dependencies from the data.}
Note that a final Batch Normalization (BN) \cite{batchnorm} layer in $f$ will make the mean of $\vz_k$ to be approximately $\beta$ (bias term in BN). This is important for equivariance as shown in \Cref{sec:equivariance}. Although we focus on a single group, \method's formulation is suited to handle multiple groups as discussed in \Cref{app:multi_group}, which we leave as future work.

\subsection{The Group Marginal Distribution}
\label{subsec:marginaldist}

As we marginalize $\vz_k$ over the content dimension ($\dc$) we get $\{\mu_j\}_{j=1}^\dg$, the sum of each column in $\vz_k$. We obtain our discretized group marginal $P(\rg | \vx_k)$ by softmaxing $\mu_j$.  
Since the parameters $g_k$ sampled during training are known, we can design a target distribution for $P(\rg | \vx_k)$ (red distribution in \Cref{fig:schema}). 
We use a von-Mises ($\VM$) target $q(\rg|\vx_k) = \VM(g_k, \kappa)$ for cyclic groups, and a Gaussian ($\mathcal{N}$) target $q(\rg|\vx_k) = \Normal(g_k, \sigma)$ for all other groups. Both targets are chosen because of the simplicity by which we can encapsulate parameter information in their structure (their mean), and the controllability of the uncertainty about $\rg$ via $\sigma$ (or $\kappa$). For readability, we refer to the uncertainty as $\sigma$ for both $\VM$ and $\mathcal{N}$ targets, where $\sigma \approx \frac{1}{\sqrt{2\pi\kappa}}$. 

To be comparable to $P(\rg|\vx_k)$, we also discretize our target in $[0, 1]$ obtaining $Q(\rg|\vx_k)$. Let $\Omega_j$ be the intervals of a $\dg$-sized partition, and $g_j$ their centers. Then, the discretized target is obtained by integrating the continous target according to the partition: $
Q_j(\rg|\vx_k) := Q(\rg=g_j\,|\,\vx_k) = \frac{\int_{\Omega_j} q(\rg|\vx_k) dg}{\int_0^1 q(\rg|\vx_k) dg}$. For the Gaussian target, we assume a slight boundary effect as we do not integrate the tails beyond $\Omega_j$.




We encourage the observed $P(\rg|\vx_k)$ to match the target by minimizing the Jensen-Shannon Divergence ($\JSD$) between the discretized distributions. The group loss for the $i$-th image $\vx^{(i)}$ in a batch is
\begin{equation}
    \label{eq:loss_g}
    \Loss_{\sG}^{(i)} = \frac{1}{2}\sum_{k=1}^2\JSD\big(P(\rg|\vx_k^{(i)}) \; \| \; Q(\rg|\vx_k^{(i)})\big).
 \end{equation}
\paragraph{The choice of $\sigma$ is key to encourage structure} 
Both very small and very large values of $\sigma$ will lead to a loss of structure in the columns of $\vz$. 
For small $\sigma$, the target takes a form close to a $\delta$ distribution. This results in an invariant discretized target (as $\delta$ moves inside interval $\Omega_j$) or abrupt target changes (as $\delta$ moves from $\Omega_j$ to $\Omega_{j+1}$), which prevents learning proper structure.
Conversely, for large $\sigma$, the target will be close to a uniform distribution, thus removing all information about the group element (all columns contribute equally). In \Cref{sec:ablation} we find empirically  that  $\sigma\approx 0.2$ is optimal in our setting. Note that this value corresponds to a normal distribution $\mathcal{N}(\cdot, 0.2)$ that covers the $[0, 1]$ domain within approximately its $3\sigma$ span (when centered at $0.5$), being a good trade-off in terms of structure.
Interestingly, since our group elements are bounded in $[0, 1]$, the value of $\sigma$ can be kept constant for all transformation groups and data sets.

\subsection{The Content Marginal Distribution}

As we marginalize $\vz_k$ by summing over the group dimension ($\dg$), we obtain  $P(\rc|\vx_k)$, the probability of observing the content $\rc$ given $\vx_k$. Such distribution is invariant to the group actions, and contains all relevant information of $\vx_k$ not related to the group $\sG$. For example, the content of an image of a horse is still a horse regardless of is rotation.
We maximize the agreement between the content of two views of an image ($\vx_1, \vx_2$). Our content representation is defined directly by the values of $P(\rc|\vx_k)$, noted as $\vc_k\in \R^{\dc}$. Following the recent trends in MSSL, both content representations are projected with a network $h$. Then we use the $\text{NTXent}$ loss \cite{simclr} in form of $\Loss_{\sC} = \text{NTXent}(h(\vc_1), h(\vc_2))$ for a SimCLR-based architecture.


\subsection{The \method Loss}

Our final loss for a full batch of $N$ images is
\begin{equation}
    \label{eq:full-loss}
    \Loss_{\text{\method}} = \frac{1}{N}\sum_{i=1}^N \Loss_{\sC}^{(i)} + \lambda \Loss_{\sG}^{(i)}.
\end{equation}
$\Loss_{\sC}$ encourages similarity between the content representations of 2 views, explicitly \textit{made} invariant to the group action, as opposed to SimCLR which contrasts representations \textit{to achieve} invariance to the group action. 
The parameter $\lambda$ controls how strongly the group structure is imposed.

\subsection{Recovering the Transformation Parameter}
\label{sec:recover}

An interesting property of \method representations is the ability to recover the transformation parameter of a test image without relying on extra regression heads. This property is useful to transform representations equivariantly (see \Cref{sec:equivariance}). It also enables interpretability, since one can analyze the default transformation parameters associatied to an image or a dataset. Assuming optimal training of $L_\sG$, the group marginal will resemble the imposed target. Therefore, the transformation parameter $\tilde{g}$ of an arbitrary image $\vx_k$ for a Gaussian target is directly recoverable as $\tilde{g} = \E[\rg|\vx_k] \approx \sum_{j=1}^\dg P_j(\rg | \vx_k) g_j$. 
In practice, for improved robustness, we fit a Gaussian (or $\VM$) function to the values $P_j(\rg|\vx_k)$, and we estimate $\tilde{g}$ as is argmax.

\subsection{Equivariance in \method}
\label{sec:equivariance}
Similarly to the approach of ESSL, \method encourages equivariance by making the neural network sensitive to the transformation parameter $g$. However, in our method this sensitivity is defined explicitly via \cref{eq:loss_g}, such that applying a transformation $\tau_g(\vx)$ in input space results in shifting by $g$ the mean of the group marginal distribution corresponding to their representation $\vz = f(\vx)$. 
In practice, we prescribe \emph{a-priori} a form for the feature-space transformation $T_g$ corresponding to the input-space transformation $\tau_g$ in \cref{eq:equivariance}, with the advantage of gaining additional controllability over such transformations (see also \cref{sec:generation}).

Specifically, we design $T_g$ according to the following considerations. Assuming optimal training of $L_\sG$, we have that the recovered group marginal distribution $P(\rg|\vx_k)$ for a given input $\vx$ transformed by $\tau_{g_k}$ resembles the target $Q(\rg|\vx_k)$. 
For \Cref{eq:equivariance} to hold, we need to design $T_g$ such that applying the column sum and softmax operations used to derive $P(\rg|\vx_k)$ (see \Cref{subsec:marginaldist}) to $T_{g_k}(\vz)$ also resembles $Q(\rg|\vx_k)$. 
We can ensure this by changing the column sums of $\vz$ (denoted as $\{\mu_j\}_{j=1}^{\dg}$) with values $\{\hat{\mu}_j\}_{j=1}^{\dg}$ that after applying the softmax yield $Q(\rg|\vx_k)$, i.e.
$\hat{\mu}_j = \text{softmax}^{-1}(\hat{Q}_j)$ with $\hat{Q}_j = {Q}_j(\rg|\vx_k)$ for ease of notation. There are infinitely many solutions for $\hat{\mu}_j$, so we choose the $\hat{\mu}_j$ that satisfies $\sum_j \hat{\mu}_j = \beta_j$. This choice comes from the fact that the final BN layer in $f$ will make the mean of $\vz$ close to the BN bias terms $\beta_j$.

Therefore,
\begin{equation}
  \hat{\mu}_j = \ln \hat{Q}_j + \ln \sum_j  e^{\hat{\mu}_j}\quad \text{with} \quad \sum_j \hat{\mu}_j = \beta_j.
\end{equation}
The solution to this equation is given by
\begin{equation}
    \hat{\mu}_j = \ln \hat{Q}_j + \beta_j - \frac{1}{\dg} \sum_j \ln \hat{Q}_j .
    \label{eq:muhat}
\end{equation}

Finally, we define $T_g$ so that it swaps the mean $\mu_j$ with $\hat{\mu}_j$
\begin{equation}
\label{eq:t_g}
    T_g(\vz) = \vz - \mM + \widehat{\mM}_{g},
\end{equation}
where all elements of each column $j$ of $\mM$ (or $\widehat{\mM}_{g}$) take the value $\mu_j$ (or $\hat{\mu}_j$).
As such, applying the column sum and softmax operations to $T_{g_k}(\vz)$ yields the same values as applying them to $\vz_k$ (at optimality), which is a necessary condition for \Cref{eq:equivariance} to hold. 
Furthermore, defined in this way, $T_g$ satisfies the group axioms (\Cref{app:proof-axioms}, again assuming $L_\sG$ is minimized).

In practice, as it also happens in other works such as ESSL or TVAE, we cannot expect \cref{eq:equivariance} to hold always (i.e. for all $\vx$ and $\rg$),
as that would require perfect generalization of the learned equivariance.
However, for our method, we can bound the equivariance generalization error w.r.t. unseen $\rg$ (\Cref{sec:equivarianceproof}), and furthermore demonstrate that it is small in practice in \Cref{sec:empiricalequivariance}.

\change{\textbf{On predictiveness and equivariance.} It is key to differentiate between predictiveness and equivariance. While predictiveness implies equivariance, the opposite is not always true (\eg invariance is a specific case of equivariance that does not imply predictiveness). Therefore, we emphasize that the approximate equivariance in \method is a by-product of the predictiveness of $\rg$ at group marginal level.}

\begin{figure*}[htb]

\caption{Empirical validation of equivariance in \method. We measure the L2 distance between the representations of a transformed image $f(\tau_g(\vx))$ and the transformed representations of that image $T_g(f(\vx))$, varying $g \in [0,1]$ along both axes. The plots show the average L2 distance for 1000 CIFAR-10 test images. Note the strong similarity for the same group element (diagonal), and the cyclic nature of rotations or flips when using a $\VM$ target (\textbf{top row}), as opposed to the Gaussian target (\textbf{bottom row}). More results in \Cref{app:empiricalequivariance}.}
\label{fig:empiricaleq}

\centering
\begin{subfigure}{0.22\textwidth}
    \centering
    \includegraphics[width=\textwidth]{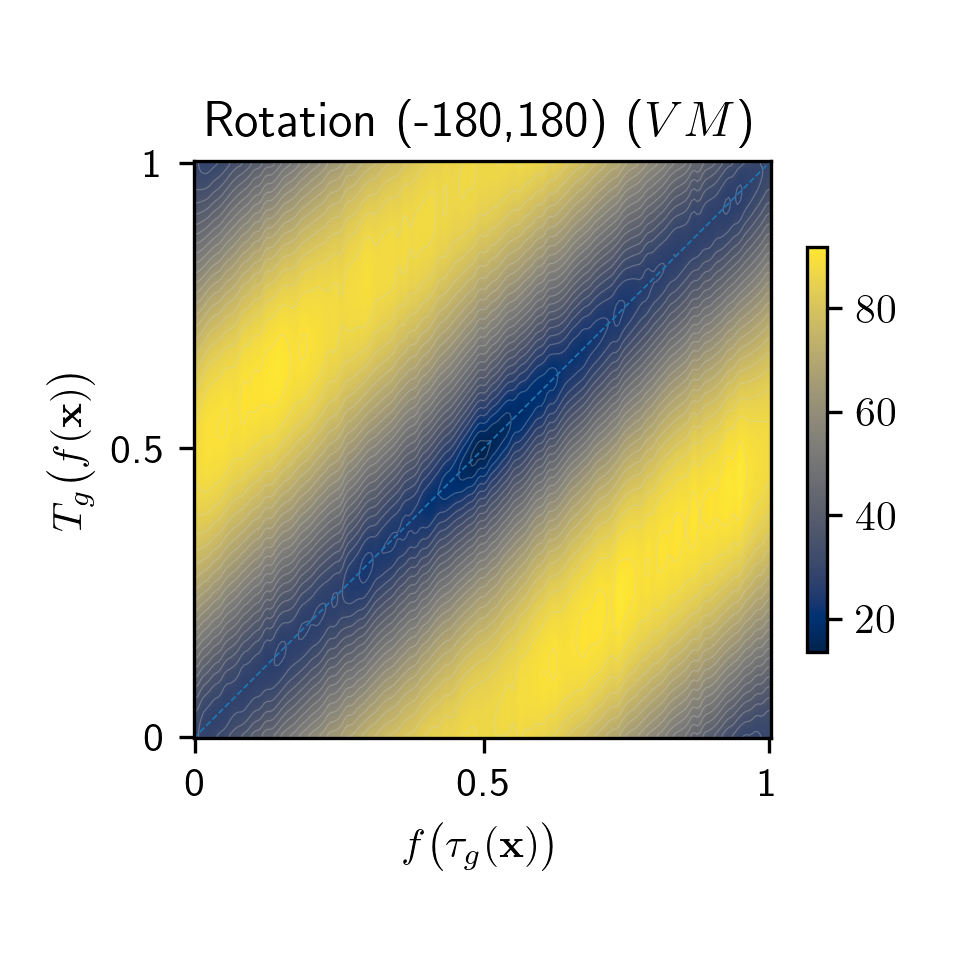}
\end{subfigure}%
\hfill
\begin{subfigure}{0.22\textwidth}
    \includegraphics[width=\textwidth]{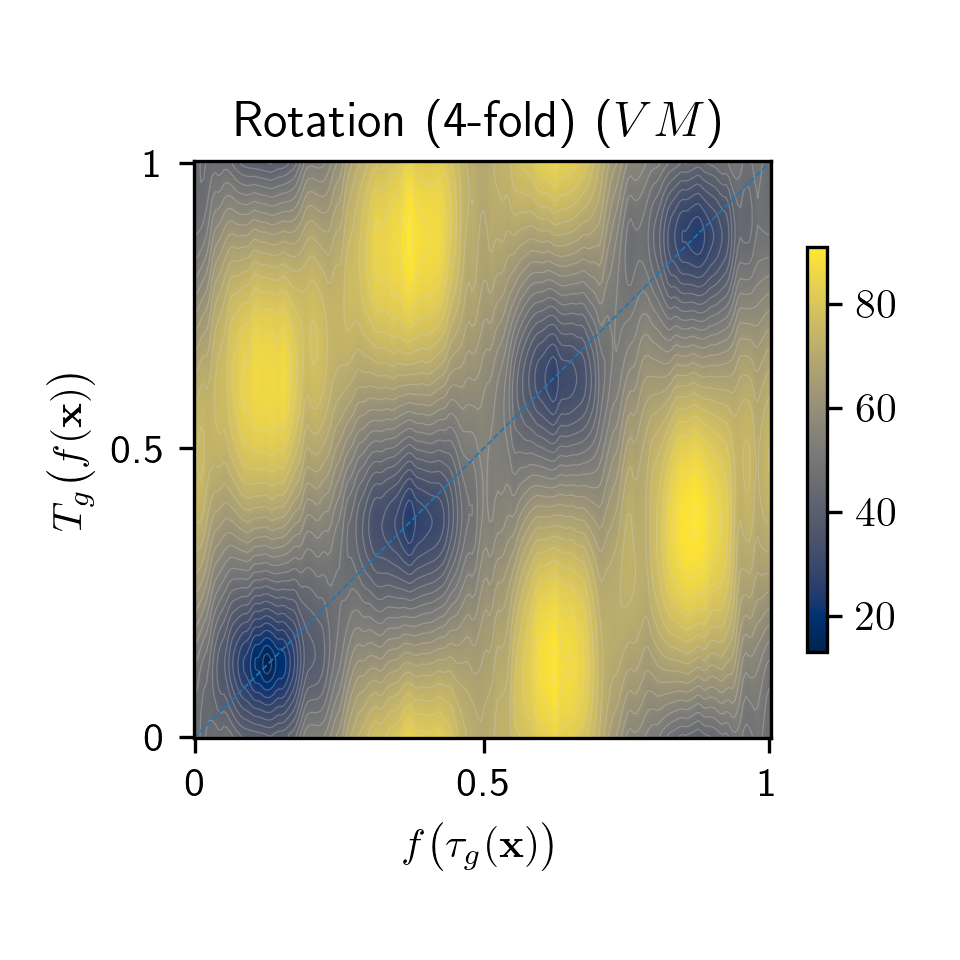}
\end{subfigure}%
\hfill
\begin{subfigure}{0.22\textwidth}
    \includegraphics[width=\textwidth]{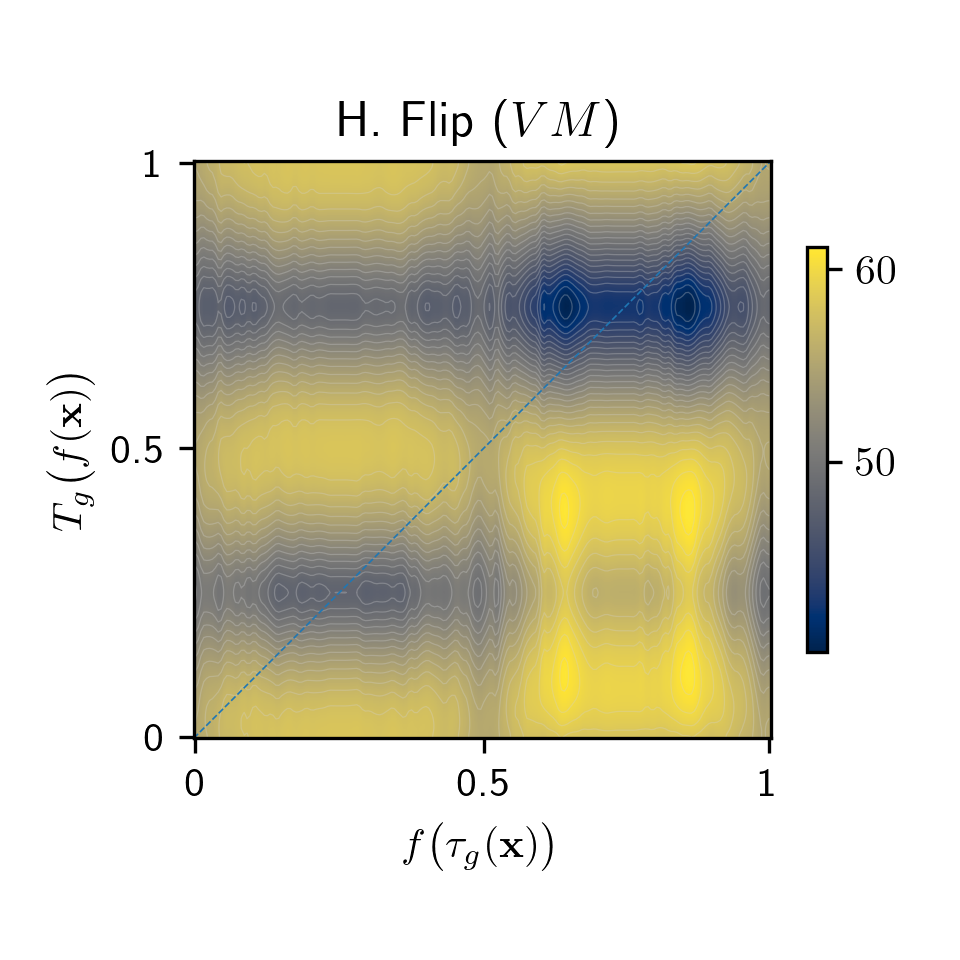}
\end{subfigure}
\hfill
\begin{subfigure}{0.22\textwidth}
    \includegraphics[width=\textwidth]{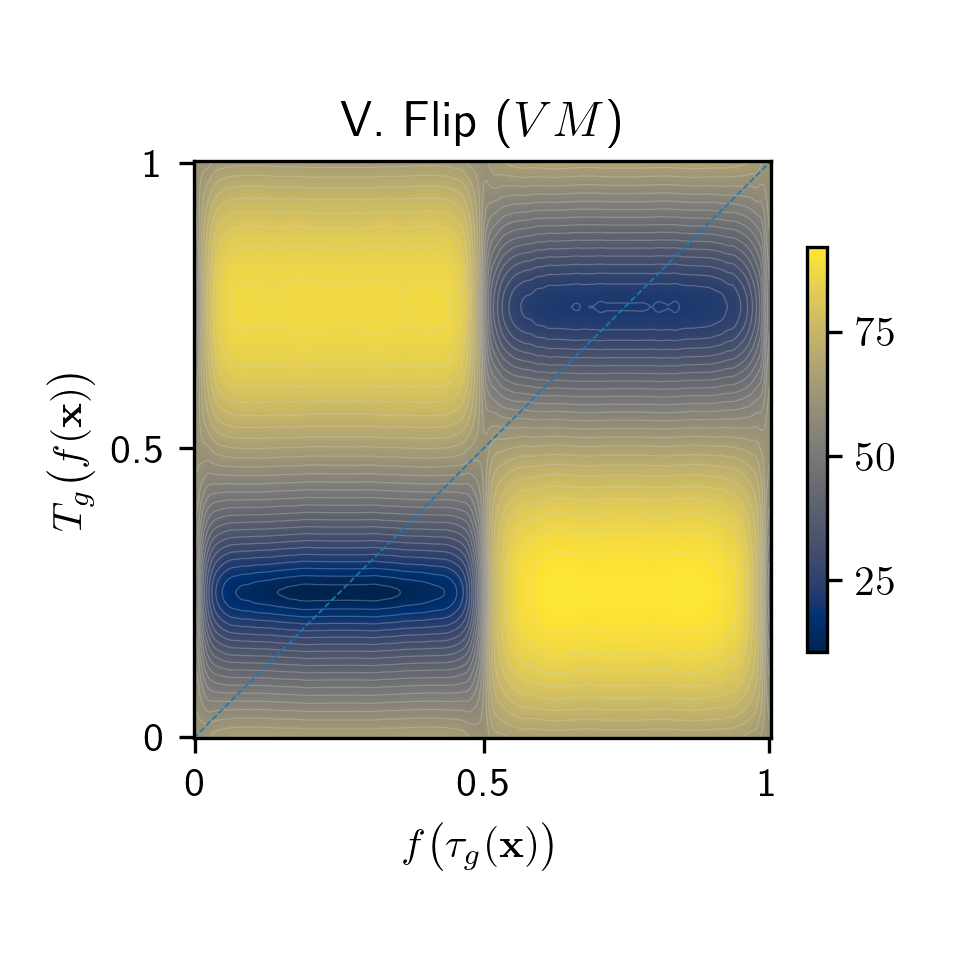}
\end{subfigure}

\vspace{-4mm}

\centering
\begin{subfigure}{0.195\textwidth}
    \includegraphics[width=\textwidth]{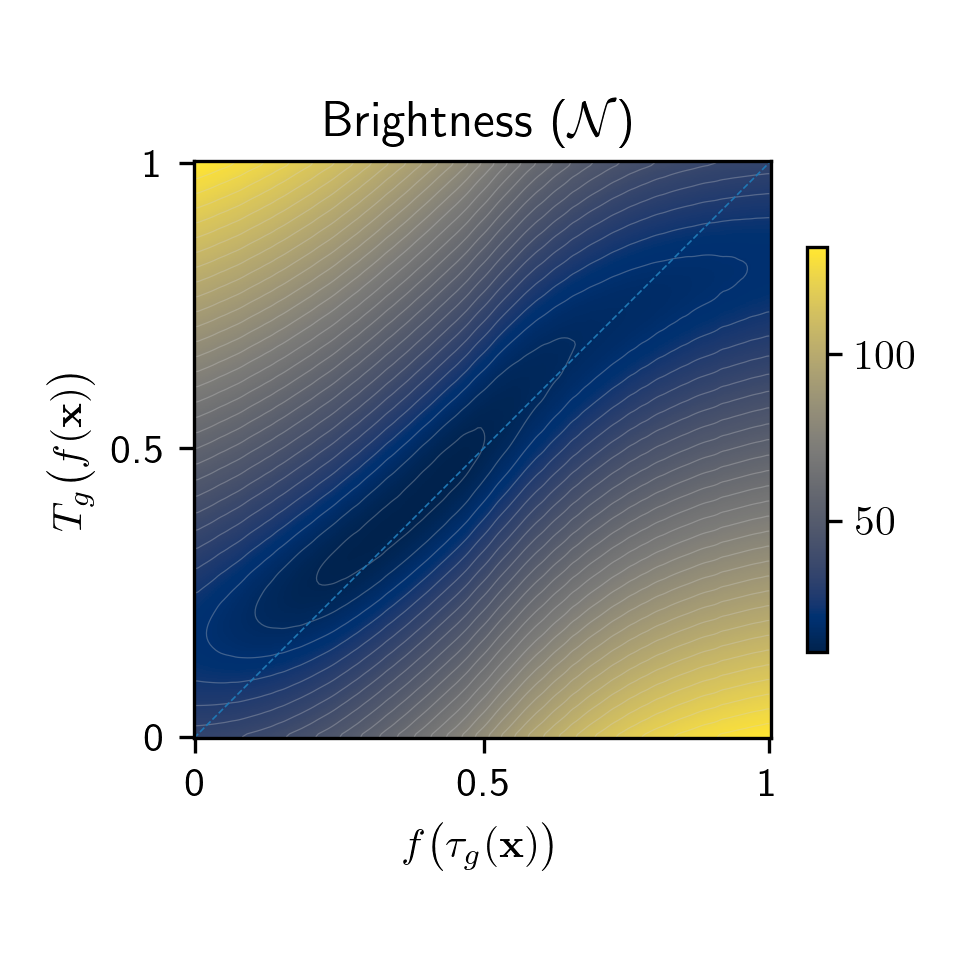}
\end{subfigure}
\hfill
\begin{subfigure}{0.195\textwidth}
    \includegraphics[width=\textwidth]{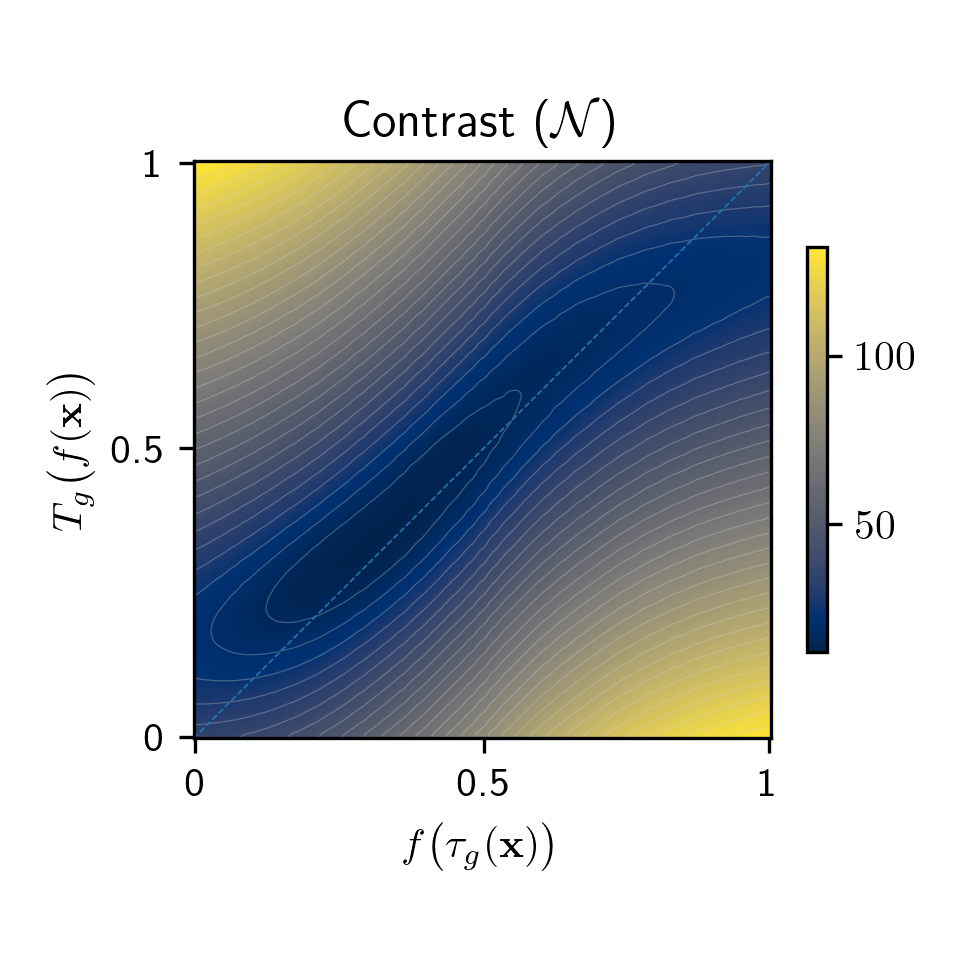}
\end{subfigure}
\hfill
\begin{subfigure}{0.195\textwidth}
    \includegraphics[width=\textwidth]{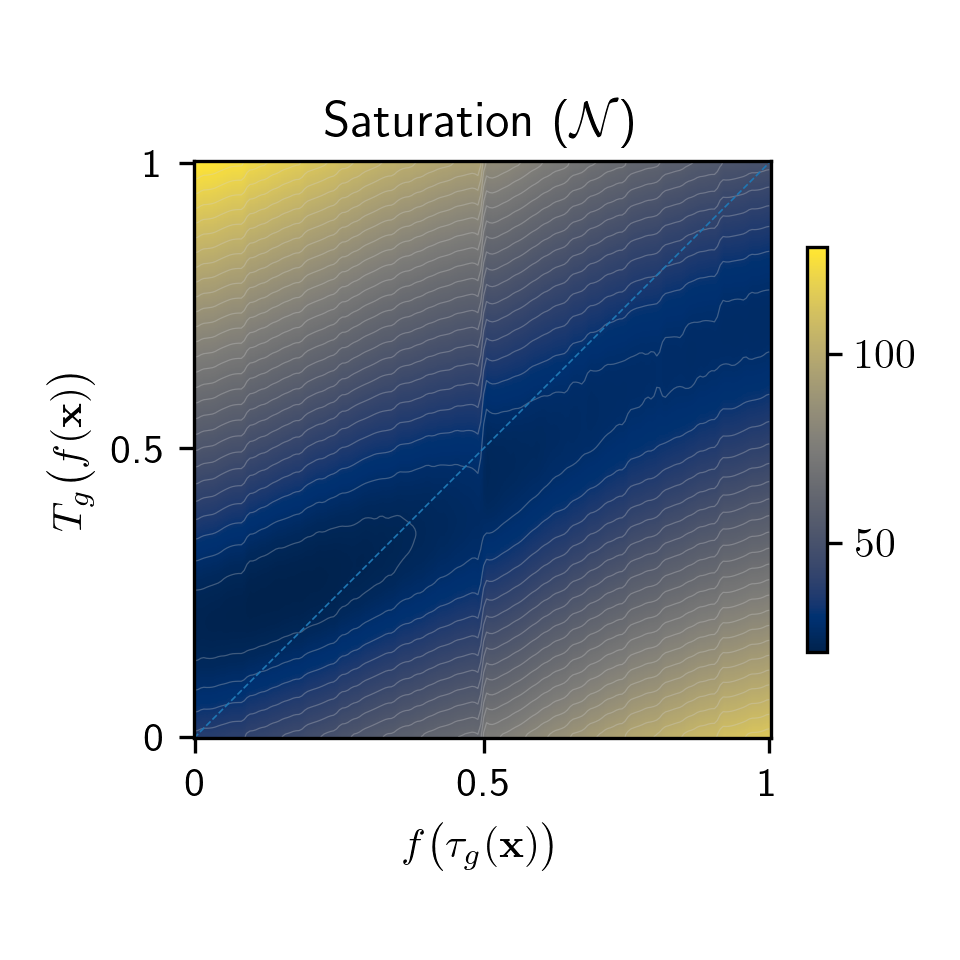}
\end{subfigure}
\hfill
\begin{subfigure}{0.195\textwidth}
    \includegraphics[width=\textwidth]{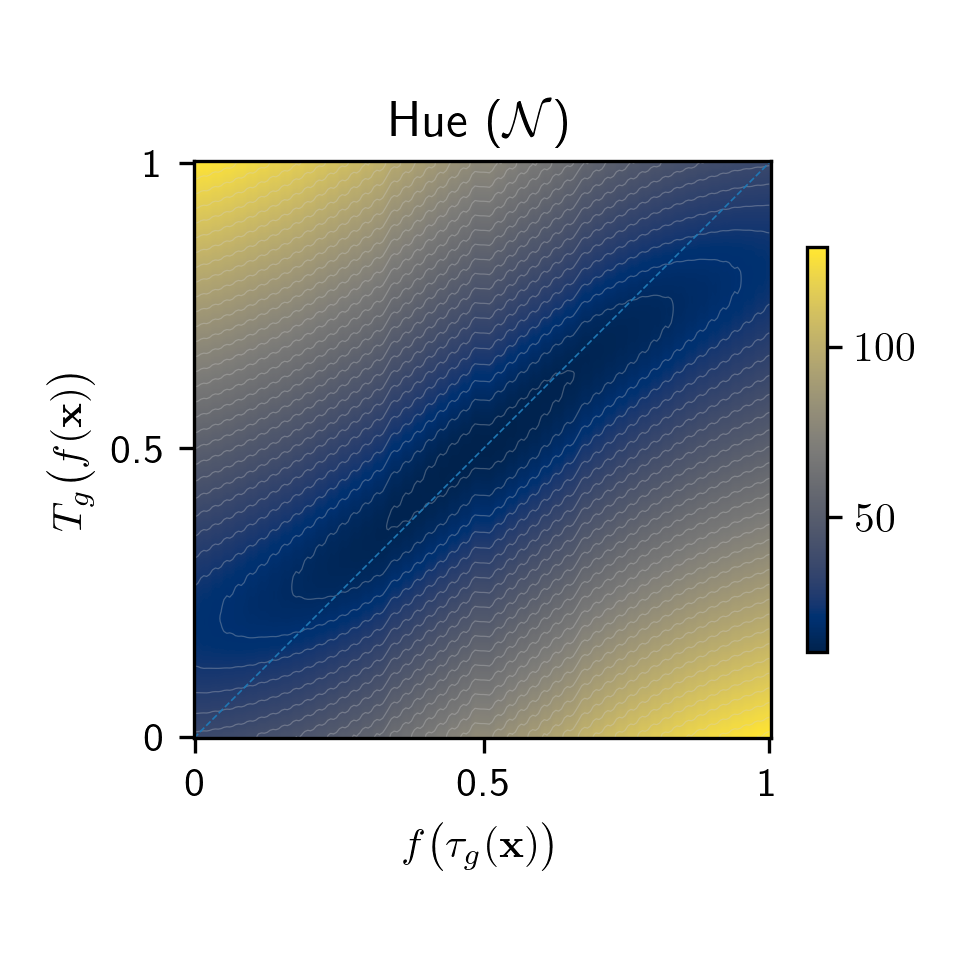}
\end{subfigure}
\hfill
\begin{subfigure}{0.195\textwidth}
    \includegraphics[width=\textwidth]{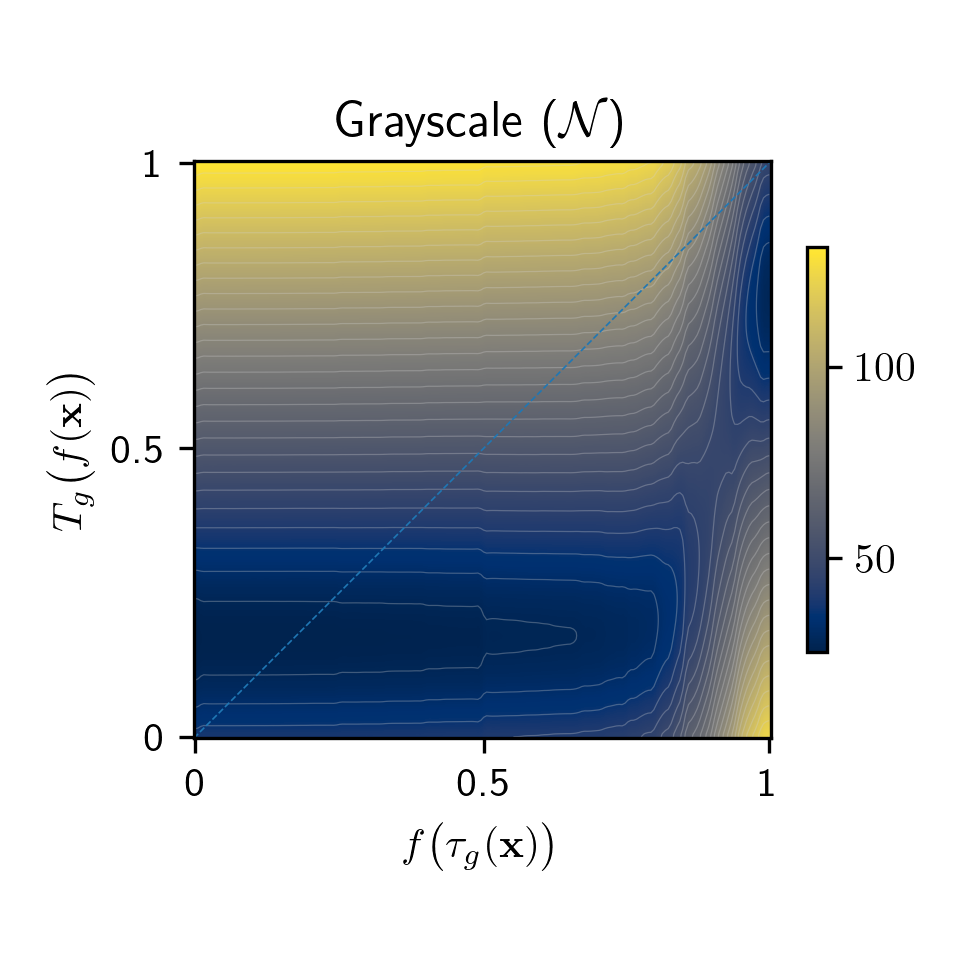}
\end{subfigure}
\vspace{-4mm}
\vskip -0.1in
\end{figure*}

\section{Experimental Results}

\subsection{Empirical Proof of Equivariance}
\label{sec:empiricalequivariance}

We start with an empirical validation of equivariance by measuring how \cref{eq:equivariance} holds for real data. To do this, we use the transformation $T_g$ from \cref{eq:t_g} and compute representations $f(\tau_{g_1}(\vx))$ and $T_{g_2}(f(\vx)) \;\forall g_1,g_2\in\sG$. An equivariant map should result in a minimal L2 distance $\ell_{g_1,g_2} = \| f(\tau_{g_1}(\vx)) - T_{g_2}(f(\vx)) \|_2^2$ when $g_1=g_2$ \todo{Careful: our notion of equivariance is linked to the group marginal: we can't hope (in general) that the equivariance condition holds at feature space $\vz$; we can, however, hope that it matches at its column-reduction $\mu$. That is to say, I think it would be more consistent if you compared $\| \sum_i f(\tau_{g_1}(\vx))_{ij} - T_{g_2}(f(\vx))_{ij} \|_2^2 = \|\mu(\tau_{g_1}(\vx)) - \mu(T_{g_2}(f(\vx)))\|_2^2$. Actually, I'm quite impressed that the feature-space error is small at all! This hints at the fact that there's some implicit regularisation going on under the hood, which forces the whole feature space to follow equivariance, too! Amazing!}. To verify this, we plot the pairwise $\ell_{g_1,g_2}$ for all elements $g_1,g_2$ and different transformations. More precisely, we sweep 100 values of $g_1, g_2$ in $[0, 1]$ for 1000 randomly selected CIFAR-10 \cite{krizhevsky2009learning} test images and we show the average pairwise L2 distance in \Cref{fig:empiricaleq}.

For infinite groups (\ie color transformations and rotation ($360$)), there is a strong similarity along the diagonal, validating \cref{eq:equivariance}. For finite groups (rotation 4-fold, flips and grayscale) we also see a strong similarity at the observed group elements. For example, rotation 4-fold shows 4 minima at the observed (normalized) angles. These plots also help to understand how the model generalizes to unseen group elements. 
Interestingly, equivariance for horizontal flip is only mildly learnt due to its ambiguity in the dataset (see \Cref{sec:ambiguity} for extended discussion). 
Indeed, flipped images appear naturally in CIFAR-10 (\eg{} cars looking to the right or left), and thus there is more ambiguity about the meaning of image flipping. Vertical flips are nicely learnt, since they do not naturally appear in data. Another interesting observation is that grayscale yields a constant representation as we reduce the saturation (horizontal axis) and then shows a sudden jump close to 1 (grayscale image). The model has learnt that, as soon as the image presents \textit{some} hint of color, it is \textit{not} grayscale, unless it is \textit{purely} grayscale. Note also that grayscale does not form an algebraic group, yet \method is still able to learn its structure.

\begin{figure}[tb]
\caption{Equivariance in \method. We encode a test image (leftmost images), transform its representation using $T_g$ (\cref{eq:t_g}) for several $g$, and then decode the transformed representations. See how transforming the representations exposes the input transformation learnt by the model, empirically proving equivariance.}
\centering
\label{fig:decoded_mnist_cifar}
\vskip 1mm
\begin{subfigure}[b]{1.0\columnwidth}
    \includegraphics[height=22.5mm]{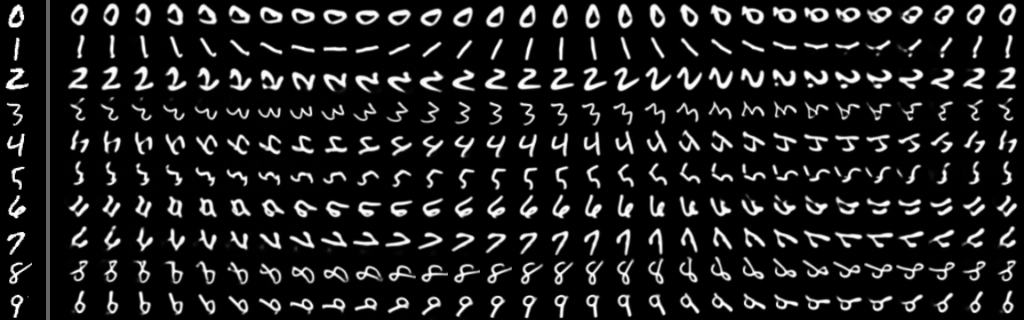}\hfill
    \includegraphics[height=22.5mm]{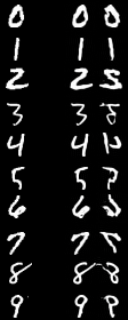}
    \caption{MNIST with \textit{Rot. (360)} (left) and \textit{horizontal flip} (right).}
\end{subfigure}
\vskip -1.7mm
\begin{subfigure}[b]{1.0\columnwidth}
    \includegraphics[width=\columnwidth]{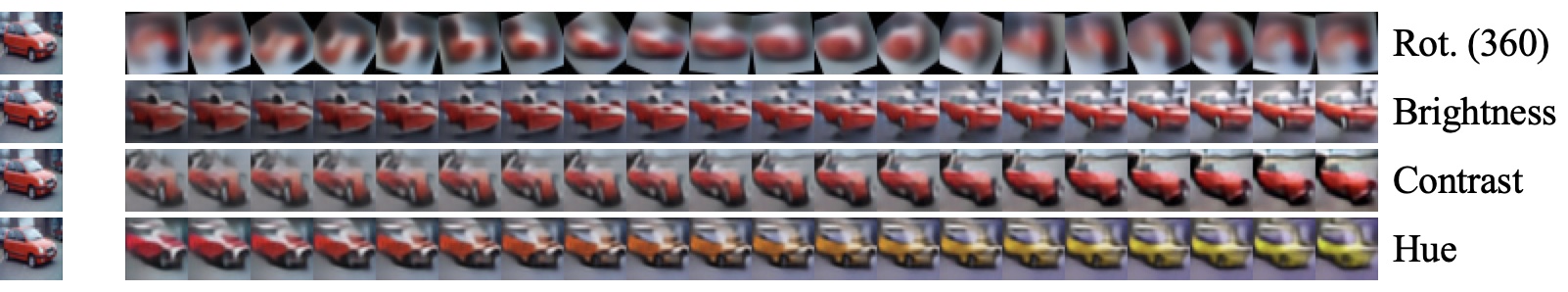}
    \caption{CIFAR-10 for \textit{Rot. (360)} and \textit{color} transformations.}
\end{subfigure}
\vskip -0.2in
\end{figure}

\subsection{\method Representations for Group Conditional Generation}
\label{sec:generation}

In \Cref{fig:decoded_mnist_cifar} we showcase the benefit of equivariance in \method representations to conditioning generation on specific group elements. 
To this end, we train a decoder on frozen pre-trained \method representations. 
In this work we do not aim to obtain state-of-the-art generation quality, but rather use a decoder for visual validation of our hypotheses. Note that group conditional generation is not feasible with MSSL methods like SimCLR or ESSL since there is no explicit transformation at representation level.

Here we exploit the equivariant property of \method representations for controlled generation. We first obtain the representation of a test image $\vz = f(\vx)$ (leftmost images in \Cref{fig:decoded_mnist_cifar}), then we create multiple transformed representations $\{T_g(\vz)\}$ using \cref{eq:t_g}, by sweeping $g$ between $0$ and $1$, and finally we decode all $\{T_g(\vz)\}$. In \Cref{fig:decoded_mnist_cifar} we show the decoded images for different datasets and groups. Notice how we can recover the input transformation by only transforming the representations, which provides yet an additional visual proof of equivariance in \method. In \Cref{app:rec_error} we show that the reconstruction error of \method is up to 66\% smaller than with SimCLR (rotation (4-fold)) and up to 70\% smaller than with ESSL (grayscale).

\subsection{\method Representations for Classification}
\label{sec:rep-cls}
\subsubsection{RRC+1 Experiments}
\label{RRC+1}

\begin{figure*}[tb]
\caption{Test top-1 performance of a linear tracking head on CIFAR-10. It can be observed that \method improves over SimCLR and ESSL for all transformations. Notably for continuous color transformations, ESSL significantly degrades performance unlike \method.}
\centering
\label{fig:cifar10_top1}
\includegraphics[width=0.9\textwidth]{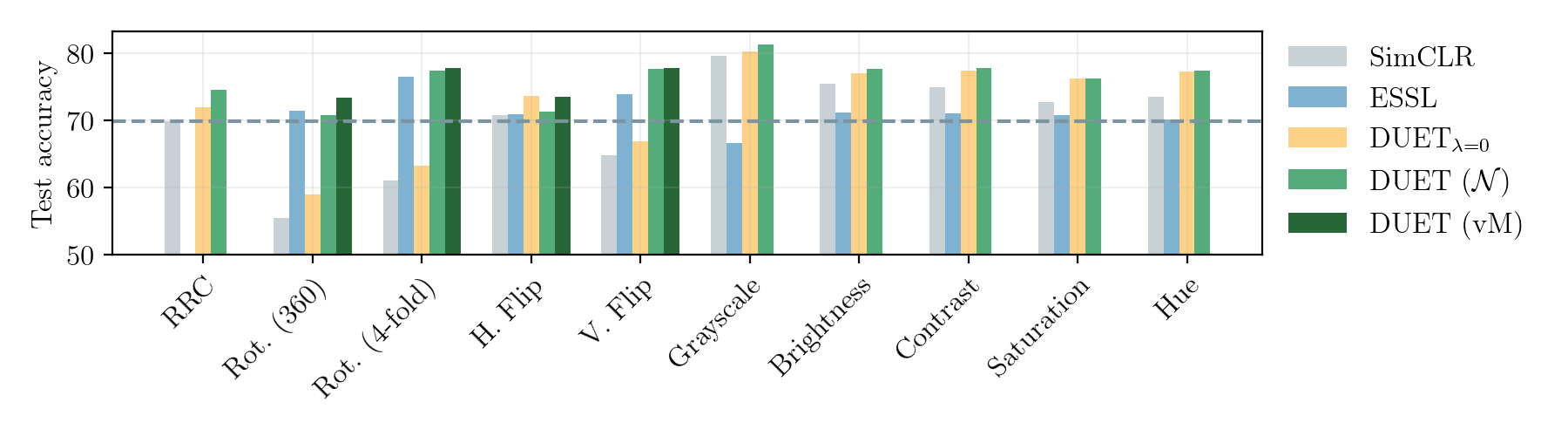}
\vskip -0.3in
\end{figure*}

In this section we analyze how \method representations perform for discriminative tasks. Following the procedure in the ESSL work, where a single transformation is applied on top of RandomResizedCrop (RRC), we carry out the set of {RRC+1} experiments. We compare our method with SimCLR and ESSL\footnote{ESSL representations are implicitly equivariant but do not guarantee interpretable structure with respect to the transformation.}.
We also compare with a variant of our method (coined \stovez) optimized without the group loss, that is with $\lambda=0$ in \Cref{{eq:full-loss}}. Notice that in \stovez we still reshape the features to 2d and sum over the columns to obtain the content representation (that is contrasted), which is a fundamental difference with SimCLR. \stovez learns unsupervised invariances very similarly to what NPTN \cite{nptn} does, but does not guarantee equivariance.
For RRC+1, \method uses $\lambda=10$, except for \textit{rotations} and \textit{vertical flip} for which we use $\lambda=1000$ according to the empirical study in \Cref{sec:ablation}. The remaining parameters are set to $\sigma=0.2$ and $\dg=8$. The full training procedure is provided in \Cref{sec:train_procedure}.

In \Cref{fig:cifar10_top1} we show the accuracy of a linear tracking head for the RRC+1 experiment on CIFAR-10. The horizontal dashed line shows the baseline performance of SimCLR with only RRC.
For all considered transformations, we show results from training with a $\mathcal{N}$ target group-marginal. For cyclic transformations (rotations and flips), we further specialize the target and consider a (periodic) $\VM$ distribution instead, reporting results for this case also.
It is important to point out that, for \method, the tracking head receives our 2d representation flattened, and as such is of the \textit{same dimensionality} as in compared methods. 

Our method outperforms SimCLR for all transformations, and even improves over SimCLR with RRC only by learning structure with respect to scale. Note that, by construction, ESSL cannot improve over the RRC-only baseline. A prominent result is the performance of \method with color transformations. For the discrete transformation \textit{grayscale}, ESSL degrades performance by 12.8\% with respect to SimCLR with grayscale, while \method improves it by 1.75\%.
For continuous color transformations, \method improves over SimCLR between 3-5\%, while ESSL degrades the performance by up to 4.3\% (brightness). This shows that the implicit equivariance in ESSL is not sufficient in this case.

We also observe in \Cref{fig:cifar10_top1} that ESSL does not improve over SimCLR for horizontal flips, while \method~($\VM$) improves by 2.8\%. In general, for ambiguous transformations like horizontal flip (see \Cref{sec:ambiguity}), we find that learning unsupervised structure with \stovez is beneficial.
We also see that structure (and equivariance) is strongly helpful for vertical flips.
For the more complex cyclic transformations, \stovez underperforms by a large margin, since such complex structure is harder to learn in a completely unsupervised way. This result shows that accounting for the topological structure of the transformation (as studied  by \citet{homeomorphic}) is of great importance, and opens the door to further research in this direction. 
Surprisingly, \stovez outperforms SimCLR. We speculate that the unsupervised structure learnt by \stovez might induce a more discriminative organization of the embedding space.

\Cref{tab:RRCp1} benchmarks the more complex tasks CIFAR-100 \cite{krizhevsky2009learning} and TinyImageNet \cite{tin}. We report the average across the cyclic groups and the color-related groups for better readability. \method achieves the highest accuracy compared to all algorithms tested, including \stovez, and across all groups but horizontal flip. \method also improves under color transformations with respect to SimCLR with the same transformations, while ESSL shows a degradation. Indeed, the datasets used in \Cref{tab:RRCp1} present higher data scarcity per class than CIFAR-10. In such setting, the structure learnt by \method shines over unstructured methods like ESSL.

\subsubsection{Full Augmentation Stack Experiments}
\label{sec:fullstack}

In this section we use the full augmentation stack as in SimCLR (see details in \Cref{app:fullstack-augs}). We learn structure for one group at a time, while applying the full stack on input images. Note that in the full stack setting, we use a fixed $\lambda=10$ for \method\footnote{We did not perform an extensive hyper-parameter tuning, the focus of this work being an exploration of structured representations in MSSL.}. We observed in this case that extremely large $\lambda$ can harm performance since multiple transformations add ambiguity to the group being learnt. 

\Cref{tab:fullstack} reports the test top-1 accuracy on CIFAR-10, CIFAR-100 and TinyImagenet. 
One interesting observation is that \method becomes better than the compared methods as the dataset complexity increases, achieving the best average accuracy for all sets of transformations on TinyImageNet. For smaller and simpler datasets like CIFAR-10, \method outperforms ESSL for color transformations, but ESSL is better for cyclic transformations. Still, \method outperforms the SimCLR baseline for cyclic transformations.

Interestingly, neither \method nor ESSL outperform SimCLR by becoming equivariant to horizontal flips, as discussed in \Cref{sec:ambiguity}. Nevertheless, \method still outperforms ESSL for horizontal flips by 0.86\%, 4.4\% and 4.18\% on CIFAR-10, CIFAR-100 and TinyImageNet respectively. Similarly, as the dataset complexity increases, \method performs better than ESSL for vertical flips.
Another interesting result is the effectiveness of ESSL with rotations, where \method remains subpar but better than the SimCLR baseline. 
The RRC column shows that \method, by just learning structure to scale (approximately, as explained in \Cref{sec:mssl_transformations}), can improve accuracy using the vanilla SimCLR augmentation stack.

It is surprising how well \stovez performs in the full stack setting, surpassing \method for simpler datasets. Indeed, \stovez learns an unsupervised structure, thus accounting for the interdependencies between the transformations applied. However, as observations of the transformation of interest are scarcer (\eg more complex datasets or less data per class) optimizing for a known structure is beneficial.

\addtolength{\tabcolsep}{-2pt}    
\begin{table*}[tb]
    \caption{RRC+1 results: Accuracy of a linear tracking head on CIFAR-100 and TinyImageNet. We also show the average over cyclic ($\VM$ target) and non-cyclic ($\mathcal{N}$ target) transformations. \method improves over SimCLR for all groups, while ESSL worsens performance for color transformations. We report the $\text{mean}_{\text{std}}$ over 3 runs.} 
    \label{tab:RRCp1}
    \scriptsize
    \vskip 0.1in
    \begin{tabularx}{\textwidth}{@{}Xlc!{\vrule width 1.5pt}cccc|c!{\vrule width 1.5pt}ccccc|c@{\hspace{0mm}}@{\hspace{0mm}}c}
    \toprule
    Dataset & Method &  RRC & Rot. (360) &   Rot. (4-fold) &         H. Flip &         V. Flip &    Avg.  &       Grayscale &              Brightness &                Contrast &             Saturation &                     Hue & Avg. \\
    \midrule

    \multirow{4}{*}{CIFAR-100}  & SimCLR               &           38.89\textsubscript{0.26} &           32.17\textsubscript{0.14} &           35.52\textsubscript{0.33} &           39.85\textsubscript{0.13} &           36.68\textsubscript{0.27} &    36.06\textsubscript{0.18}   &    47.41\textsubscript{0.40} &           45.00\textsubscript{0.26} &           44.51\textsubscript{0.04} &           42.27\textsubscript{0.70} &           43.61\textsubscript{0.19} &           44.56\textsubscript{0.26} \\
& ESSL                 &                        - &           38.03\textsubscript{0.36} &           44.36\textsubscript{0.72} &           38.78\textsubscript{0.54} &           42.17\textsubscript{0.87} &  40.84\textsubscript{0.51} &         34.20\textsubscript{0.60} &           38.33\textsubscript{0.25} &           38.66\textsubscript{0.02} &           38.92\textsubscript{0.08} &           37.64\textsubscript{0.39} &           37.55\textsubscript{0.22} \\
& \stovez   &           42.63\textsubscript{0.11} &           34.77\textsubscript{0.72} &           38.28\textsubscript{0.19} &  \textbf{43.54}\textsubscript{0.67} &           40.10\textsubscript{0.82} &  39.17\textsubscript{0.49} &           48.87\textsubscript{0.18} &           48.12\textsubscript{0.19} &           47.74\textsubscript{0.57} &           45.39\textsubscript{0.34} &           46.32\textsubscript{0.61} &           47.29\textsubscript{0.31} \\
& \method  &  \textbf{45.25}\textsubscript{0.10} &           \textbf{42.17}\textsubscript{0.42} &  \textbf{47.25}\textsubscript{0.26} &           41.82\textsubscript{0.46} &           \textbf{45.38}\textsubscript{0.73} &  \textbf{44.16}\textsubscript{0.38} &   \textbf{50.91}\textsubscript{0.49} &  \textbf{50.18}\textsubscript{0.45} &  \textbf{49.77}\textsubscript{0.46} &  \textbf{48.75}\textsubscript{0.22} &  \textbf{48.54}\textsubscript{0.78} & \textbf{49.63}\textsubscript{0.39} \\

    \midrule
    \multirow{4}{*}{TinyImageNet}   & SimCLR               &           26.91\textsubscript{0.13} &           21.34\textsubscript{0.08} &           24.40\textsubscript{0.24} &           27.90\textsubscript{0.37} &           26.74\textsubscript{0.30} &    25.09\textsubscript{0.20} &       31.35\textsubscript{1.29} &           29.95\textsubscript{0.14} &           29.68\textsubscript{0.18} &           28.60\textsubscript{0.29} &           28.20\textsubscript{0.40} &           29.55\textsubscript{0.38}  \\
& ESSL                 &                        - &           25.35\textsubscript{0.10} &           30.41\textsubscript{0.25} &           27.13\textsubscript{0.14} &        29.11\textsubscript{0.27} &      28.00\textsubscript{0.16} &     23.51\textsubscript{0.18} &           26.45\textsubscript{0.07} &           26.32\textsubscript{0.61} &           26.75\textsubscript{0.72} &           26.00\textsubscript{0.11} &           25.80\textsubscript{0.28}  \\
& \stovez   &           29.57\textsubscript{0.47} &           24.22\textsubscript{0.30} &           26.96\textsubscript{0.37} &  \textbf{30.78}\textsubscript{0.23} &           28.98\textsubscript{0.42} &   27.73\textsubscript{0.27} &        31.55\textsubscript{0.14} &           32.54\textsubscript{0.18} &           32.23\textsubscript{0.17} &           31.43\textsubscript{0.24} &           30.45\textsubscript{0.60} &           31.64\textsubscript{0.22} \\
& \method &  \textbf{31.26}\textsubscript{0.21} &           \textbf{27.78}\textsubscript{0.24} &  \textbf{31.55}\textsubscript{0.28} &           30.34\textsubscript{0.59} &           \textbf{31.71}\textsubscript{0.53} & \textbf{30.34}\textsubscript{0.33} & \textbf{34.92}\textsubscript{0.24} &  \textbf{33.96}\textsubscript{0.16} &  \textbf{34.20}\textsubscript{0.34} &  \textbf{33.42}\textsubscript{0.35} &  \textbf{32.64}\textsubscript{0.03} &  \textbf{33.83}\textsubscript{0.18} \\

    \bottomrule
\end{tabularx}
\end{table*}
\addtolength{\tabcolsep}{1pt}

\addtolength{\tabcolsep}{-1pt}    
\begin{table*}[tb]
    \caption{Full Stack results. We show the average accuracy of a linear tracking head over cyclic ($\VM$ target) and non-cyclic ($\mathcal{N}$ target) transformations. As the task complexity increases, \method achieves better accuracy than the compared methods. Columns \textit{Rot. (360)}, \textit{Rot. (4-fold)} and \textit{V. Flip} require an additional transformation. We report the $\text{mean}_{\text{std}}$ over 3 runs.}
    \label{tab:fullstack}
    \scriptsize
    \vskip 0.1in
    \begin{tabularx}{\textwidth}{@{}Xlc!{\vrule width 1.5pt}cccc|c!{\vrule width 1.5pt}ccccc|c@{\hspace{0mm}}@{\hspace{0mm}}c}
    \toprule
    Dataset & Method &  RRC & Rot. (360) &   Rot. (4-fold) &         H. Flip &         V. Flip &   Avg. &            Grayscale &              Brightness &                Contrast &             Saturation &                     Hue & Avg. \\
    \midrule
    \multirow{4}{*}{CIFAR-10} & SimCLR               &  87.42\textsubscript{0.01} &           79.90\textsubscript{0.50} &           81.50\textsubscript{0.44} &  87.48\textsubscript{0.06} &           82.78\textsubscript{0.23} & 82.92\textsubscript{0.22} & 87.41\textsubscript{0.03} &  \textbf{87.51}\textsubscript{0.11} &  \textbf{87.57}\textsubscript{0.19} &  87.49\textsubscript{0.08} &           87.67\textsubscript{0.34} &           \textbf{87.53}\textsubscript{0.11}\\
& ESSL                 &               - &  \textbf{86.55}\textsubscript{0.13} &  \textbf{89.33}\textsubscript{0.32} &  84.78\textsubscript{0.40} &  \textbf{86.66}\textsubscript{0.21} &  \textbf{86.83}\textsubscript{0.22} & 83.59\textsubscript{0.43} &  85.78\textsubscript{0.25} &  86.31\textsubscript{0.16} &  87.12\textsubscript{0.30} &           86.39\textsubscript{0.44} &  85.84\textsubscript{0.26} \\
& \stovez   &  \textbf{87.50}\textsubscript{0.20} &           79.05\textsubscript{0.37} &           81.32\textsubscript{0.18} &  \textbf{87.73}\textsubscript{0.19} &           82.66\textsubscript{0.14} & 82.69\textsubscript{0.22} &  \textbf{87.69}\textsubscript{0.17} &           87.47\textsubscript{0.13} &           87.34\textsubscript{0.20} &  \textbf{87.54}\textsubscript{0.33} &           87.63\textsubscript{0.20} &           \textbf{87.53}\textsubscript{0.20} \\
& \method  &  87.22\textsubscript{0.10} &           81.70\textsubscript{0.30} &           83.49\textsubscript{0.16} &  85.64\textsubscript{0.08} &           83.84\textsubscript{0.22} &  83.67\textsubscript{0.15} & 87.40\textsubscript{0.08} &  86.97\textsubscript{0.22} &  87.05\textsubscript{0.37} &  87.52\textsubscript{0.19} &  \textbf{87.97}\textsubscript{0.11} &           87.38\textsubscript{0.16} \\

    \midrule
    \multirow{4}{*}{CIFAR-100}  & SimCLR               &           61.40\textsubscript{0.17} &           56.40\textsubscript{0.30} &           57.32\textsubscript{0.03} &  61.48\textsubscript{0.28} &           56.73\textsubscript{0.48} &     57.98\textsubscript{0.19} &       61.43\textsubscript{0.21} &           61.31\textsubscript{0.05} &  61.68\textsubscript{0.57} &           61.57\textsubscript{0.41} &           61.30\textsubscript{0.03} &          61.46\textsubscript{0.18}  \\
& ESSL                 &                        - &  \textbf{58.32}\textsubscript{0.06} &  \textbf{63.28}\textsubscript{0.28} &  55.22\textsubscript{0.29} &           57.18\textsubscript{0.18} &     \textbf{58.50}\textsubscript{0.17} &       55.10\textsubscript{0.47} &           57.92\textsubscript{0.38} &  58.06\textsubscript{0.58} &           60.25\textsubscript{0.34} &           58.91\textsubscript{0.30} &          58.05\textsubscript{0.34} \\
& \stovez   &           62.13\textsubscript{0.13} &           55.49\textsubscript{0.22} &           57.79\textsubscript{0.40} &  \textbf{62.25}\textsubscript{0.34} &           56.88\textsubscript{0.18} & 58.10\textsubscript{0.29} &  \textbf{62.32}\textsubscript{0.26} &  \textbf{62.39}\textsubscript{0.27} &  \textbf{62.47}\textsubscript{0.16} &           62.54\textsubscript{0.20} &           62.29\textsubscript{0.29} &           62.40\textsubscript{0.24}  \\
& \method  &  \textbf{62.17}\textsubscript{0.28} &           55.66\textsubscript{0.39} &           58.01\textsubscript{0.31} &  59.62\textsubscript{0.11} &           \textbf{57.40}\textsubscript{0.15} & 57.67\textsubscript{0.20} &  {62.18}\textsubscript{0.51} &  {62.24}\textsubscript{0.31} &  61.90\textsubscript{0.72} &  \textbf{62.67}\textsubscript{0.19} &  \textbf{63.31}\textsubscript{0.21} & \textbf{62.46}\textsubscript{0.32}   \\

    \midrule
    \multirow{4}{*}{TinyImageNet}   & SimCLR               &           42.16\textsubscript{0.16} &           37.35\textsubscript{0.19} &           39.23\textsubscript{0.15} &  {42.31}\textsubscript{0.06} &    39.35\textsubscript{0.09} &            38.50\textsubscript{0.09} &           42.11\textsubscript{0.23} &           42.32\textsubscript{0.08} &           42.34\textsubscript{0.10} &           42.27\textsubscript{0.01} &           42.46\textsubscript{0.27} &           42.30\textsubscript{0.10}  \\
& ESSL                 &                        - &           37.53\textsubscript{0.21} &  \textbf{42.86}\textsubscript{0.29} &           36.25\textsubscript{0.13} &           37.18\textsubscript{0.77} &   38.46\textsubscript{0.29} &             35.50\textsubscript{0.30} &           37.94\textsubscript{0.13} &           38.66\textsubscript{0.50} &           40.55\textsubscript{0.74} &           40.49\textsubscript{0.05} &           38.63\textsubscript{0.28}  \\
& \stovez  &           43.07\textsubscript{0.11} &           36.28\textsubscript{0.95} &           39.54\textsubscript{0.41} &  \textbf{42.43}\textsubscript{0.33} &           38.87\textsubscript{0.40} &  39.28\textsubscript{0.52} &      \textbf{42.79}\textsubscript{0.16} &           42.61\textsubscript{0.34} &           42.98\textsubscript{0.08} &           42.86\textsubscript{0.30} &           42.90\textsubscript{0.46} &           42.83\textsubscript{0.27}  \\
& \method  &  \textbf{43.56}\textsubscript{0.54} &   \textbf{38.06}\textsubscript{0.06} &           40.03\textsubscript{0.28} &           40.43\textsubscript{0.21} &  \textbf{39.36}\textsubscript{0.50} &      \textbf{39.47}\textsubscript{0.18} &          42.55\textsubscript{1.27} &  \textbf{43.41}\textsubscript{0.01} &  \textbf{43.71}\textsubscript{0.07} &  \textbf{44.13}\textsubscript{0.57} &  \textbf{44.61}\textsubscript{0.10} &  \textbf{43.68}\textsubscript{0.29} \\

    \bottomrule
\end{tabularx}
\end{table*}
\addtolength{\tabcolsep}{1pt}

\subsection{Transfer to Other Datasets}
\label{sec:transfer}


\method's structure to rotations yields a gain of +21\% with respect to SimCLR when transferring to Caltech101 \cite{caltech101}, and between +5.97\% and +16.97\% when transferring to other datasets like CIFAR-10, CIFAR-100, DTD \cite{dtd} or Oxford Pets \cite{pets}. Structure to color transformations also proves beneficial, with a +6.36\% gain on Flowers \cite{flowers} (grayscale), Food101 \cite{food101} (hue) and +7.13\% on CIFAR-100 (hue). Horizontal flip is the transformation that sees less gain due to its ambiguity, as discussed in \Cref{sec:ambiguity}.

\section{Discussion and Limitations}

\paragraph{On the Dimensionality of \method Representations.}
We reshape the output of the backbone ($\R^D$) to $\vz \in \R^{\dc \times \dg}$. The final representation used for downstream tasks is a flattened ($\R^D$) version of $\vz$. For a fair comparison, SimCLR and ESSL also yield $\R^D$ representations.

\paragraph{Trading off Structure and Expressivity.} By increasing $\dg$ we reduce the effective dimensionality of the content representations ($\R^C$) contrasted through $L_{\sC}$. This implies a trade-off between structure (improves generation, transferrability) and expressivity (improves discrimination). 
Such effect is visible in the transfer learning results, where learning structure to rotation is not useful when transferring to the Flowers dataset. Indeed, such dataset contains many circular flowers, which
are rotation (and flip) invariant.
\paragraph{Transformation Ambiguity.}
\label{sec:ambiguity}
A dataset containing examples related by input transformations results in \emph{transformation ambiguity}, and the distribution over group actions 
$P(\rg|\vx_k)$
becomes multi-modal. This is shown in \Cref{fig:p_g_flips}
where the weight of each mode corresponds to the observed probability in the dataset, i.e., 
$P(\rg|\vx_k)$
reflects the bias of the dataset with respect to the transformation. Additional results in \Cref{app:p_g} show ambiguity also for color transformations, \eg natural images may present a different default hue, yielding a spread $P(\rg|\vx_k)$.
This phenomenon is also observed in \Cref{sec:fullstack} and \Cref{sec:transfer}, where the notion of a left-flipped image is ambiguous, whereas a vertically flipped image is not, and only the latter transformation yielded a performance gap between equivariant and invariant methods.

\begin{figure}[tb]
\caption{Observed $P(\rg | \vx)$ for horizontal (left) and vertical (right) flips, obtained from 1000 CIFAR-10 images. Note the inherent ambiguity for horizontal flips. Also, see that the modes of the distributions correspond to the mapped points specified in \Cref{tab:groups}.}
\centering
\label{fig:p_g_flips}
\vskip 1mm
\begin{subfigure}[b]{1.0\columnwidth}
    \includegraphics[width=0.48\columnwidth]{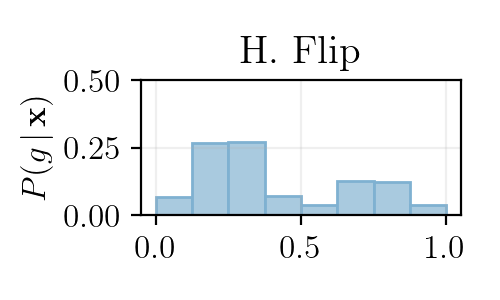}\hfill
    \includegraphics[width=0.48\columnwidth]{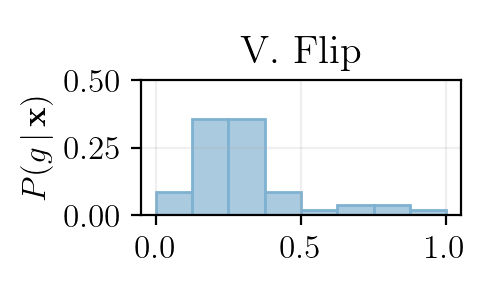}
\end{subfigure}
\vskip -0.3in
\end{figure}


\paragraph{Are $\rc$ and $\rg$ Dependent?}
To better understand the dependency between $\rc$ and $\rg$ (conditioned on $\vx_k$) quantitatively, we measure the difference $\Delta P = \| P(\rc, \rg | \vx_k) - P(\rc | \vx_k)P( \rg | \vx_k)\|_2^2$. In \Cref{tab:dependency_c_g} we report the average difference for the \method representations of 100 images from CIFAR-10, for 100 images with independent and identically distributed (iid) pixels and for 100 random representations (iid features). Note that such difference is expected to be 0 for the random representations (independent) and close to 0 for the iid pixels (no symmetries in the data).

\begin{table}[tb]
    \caption{Dependence of $\rc$ and $\rg$ conditioned on $\vx_k$. The learnt marginal representations for content ($\rc$) and group element ($\rg$) are dependent. This is a core strength of \method, where group structure and content are not assumed independent, but rather with specific dependencies learnt from data.} 
    \label{tab:dependency_c_g}
    \footnotesize
    \centering
    \vskip 0.1in
    \begin{tabular}{lc}
    \toprule
     & $\Delta P$ \\
    \midrule
    \vspace{1mm}
    \method w/ CIFAR-10 & 178.17 \\
    \method w/ iid pixels & 0.015 \vspace{1mm} \\
    iid representations & 0.00075 \\
    \bottomrule
\end{tabular}
\end{table}

\paragraph{Computational Requirements}
The training time of \mbox{SimCLR} and \method are practically the same. \method's extra requirements suppose a negligible overhead, namely: are a sum over rows and cols of $\vz$ and the computation of the Jensen-Shannon Divergence in $L_{\sC}$. Interestingly, the projection head $h$ in \method is smaller than in SimCLR, since the content features are of lower dimension, effectively reducing the model parameters with respect to SimCLR.

Compared to ESSL, DUET shows an important computational gain. Indeed, the time required for ESSL to train depends on the group chosen. Taking the implementation in \cite{essl} for 4-fold rotations, the backbone consumes $2+4$ versions of each image, resulting in an overall training time $2.01\times$ longer than that of \method. For other transformations, ESSL requires $2+2$ images being consumed (\eg flips) or $2+1$ (\eg contrast); thus resulting in longer training time than DUET in all cases.

\section{Conclusion}

We introduce \method, a method to learn structured and equivariant representations using MSSL. \method uses 2d  representations that model the joint distribution between input content and the group element acting on the input. 
\method representations, optimized through the content and group element marginal distributions, become structured and equivariant to the group elements. 
We design an explicit form of transformation at representation level that allows exploiting equivariance for controlled generation. 
Our results show that \method representations are expressive for generative purposes (lower reconstruction error) and also for discriminative purposes. 
Overall, this work shows that accounting for the topological structure of input transformations is of great importance to improve generalization in MSSL.
\section{Acknowledgements}
\label{sec:acknowledgements}

We thank
Adam Goli\'{n}ski,
Eeshan Gunesh Dhekane,
Pau Rodríguez López,
Miguel Sarabia del Castillo,
Josh Susskind,
Tatiana Likhomanenko and
Russ Webb
for their helpful feedback and critical discussions throughout the process of writing this paper; as well as 
Nick Apostoloff and 
Jerremy Holland
for supporting this research.
Names are in alphabetical order by last name within group.

\bibliography{egbib}
\bibliographystyle{icml2023}

\newpage
\appendix
\onecolumn
\begin{section}{Bounding the Equivariance Error}
\label{sec:equivarianceproof}

To introduce some notation, let us assume that, for an input data point $\vx_0$, the training procedure has ``seen'' the augmentations $\vx_i=\tau_{g_i}(\vx_0)$, generating the respective representations $\vz_i=f(\vx_i)$ in feature space. Notice that, with some abuse of notation
, in this scenario we consider $\vz$ to be the column reduction of the feature space, since that is the only part dedicated to guaranteeing equivariance. Since we are at the optimum, these must produce group marginal distributions $Q_j(z_i)\equiv \hat{Q}_{g_i}$, where $\hat{Q}_{g_i}$ represents the discretization of the target distribution with mean $g_i$. At the end of training, if exact equivariance is reached (i.e. if $L_\sG$ is minimized), a newly generated augmentation $\vx=\tau_g(\vx_0)$ for a given transformation parameter $g$, would be mapped by our neural network to the feature vector $\vz$, such that $Q(\vz)\equiv\hat{Q}_{g}$. Since this augmentation was not seen during training time, however, this is not guaranteed. We are interested in providing a bound on the error between the representation actually recovered, and the ideal one, which gives us an indication of how much our neural network can violate equivariance for unseen transformation parameters $g$. This is given by the following theorem.
\begin{theorem}
\label{thm:equivariancebound}
    For a training point $\vx_0$, at the optimum of $L_\sG = 0$, the equivariance error of a neural network $f$ trained with loss \Cref{eq:loss_g} is bounded by
    \begin{equation}
        \|f(\tau_g(\vx_0)) - T_g(f(\vx_0)) \| \leq (L_f L_{\tau_g} + L_{T_g})\min_i|g-g_i|,
    \end{equation}
    where $L_{T_g}$, $L_{\tau_g}$ and $L_{f}$ are the Lipschitz constants associated with the transformations $T_g$, $\tau_g$, and the network $f$, respectively.
\end{theorem}

\begin{proof}
Using triangular inequality, we get
\begin{equation}
    \|f(\tau_g(\vx_0)) - T_g(f(\vx_0)) \| \leq \|f(\tau_g(\vx_0)) - f(\tau_{g_i}(\vx_0))\| + \|f(\tau_{g_i}(\vx_0)) - T_g(f(\vx_0))\|
\end{equation}
for any given augmentation $\vx_i=\tau_{g_i}(\vx_0)$ seen during training time. At the optimum, we have by construction that $f(\tau_{g_i}(\vx_0)) = T_{g_i}(f(\vx_0))$, which allows us to rewrite the second term as 
\begin{equation}
    \|f(\tau_{g_i}(\vx_0)) - T_g(f(\vx_0)) \| = \|T_{g_i}(f(\vx_0)) - T_g(f(\vx_0)) \| \leq L_{T_g} |g-g_i|
\end{equation}    
Notice $L_{T_g}$ depends on the target discretization chosen: for the Gaussian target and $\infty$-norm, we recover it analytically in Lemma~\ref{thm:LipschitzTg}.
The first term, instead, becomes
\begin{equation}
\|f(\tau_g(\vx_0)) - f(\tau_{g_i}(\vx_0))\| \leq L_f\|\tau_g(\vx_0) - \tau_{g_i}(\vx_0)\| \leq L_f L_{\tau_g} |g-g_i|.
\end{equation}
Combining these results together, we recover the target bound.
\end{proof}
Notice that for a \emph{discrete} group, instead, it is possible to train $f(\vx)$ so that it achieves exact equivariance:
\begin{corollary}
Given a discrete group $\mathcal{G}$, a neural network $f(\vx)$ trained with loss \Cref{eq:loss_g} achieves equivariance at the optimum, if it is exposed to all group transformations.
\end{corollary}
\begin{proof}
    The proof follows directly from Theorem~\ref{thm:equivariancebound} by noticing that $g-g_i=0$ necessarily, if all group transformations have been seen during training time.
\end{proof}

\begin{theorem}
    \label{thm:LipschitzTg}
    For a Gaussian target, $\hat{Q}_i(g)=\frac{\int_{\Omega_i}\mathcal{N}(g,\sigma)(\tilde{g})\,d\tilde{g}}{\int_{[0,1]}\mathcal{N}(g,\sigma)(\tilde{g})\,d\tilde{g}}$, the Lipschitz continuity constant for $T_g$ in $\infty$-norm is given by $\hat{\mu}_{\dg-1}'(0)$, with $\hat{\mu}_j$ defined in \Cref{eq:muhat}.
\end{theorem}
\begin{proof}
Starting from the definition of $T_g(\vz)$ in \Cref{eq:t_g}, and using the mean-value theorem, we get
\begin{equation}
\begin{split}
\| T_g(\vz) - T_{\hat{g}}(\vz)\|_{\infty} &= \| \hat{\mM}_g - \hat{\mM}_{\hat{g}}\|_{\infty}  =  \max_{j} |\hat{\mu}_j(g) - \hat{\mu}_j(\hat{g})| = \max_{j} |\hat{\mu}_j'(\tilde{g}_j)||g-\hat{g}|\\
\end{split}
\end{equation}
for some (possibly different for different $j$) $\tilde{g}_j\in[g,\hat{g}]$. We remind that $\hat{\mu}_j(g)$ is defined in \eqref{eq:muhat} as
\begin{equation}
\begin{split}
    \hat{\mu}_j(g) &= \ln{\hat{Q}_j(g)}-\frac{1}{\dg}\sum_i\ln{\hat{Q}_i(g)} = \ln\left(\frac{\Delta\Phi_j^{j+1}(g)}{\Delta\Phi_0^{\dg}(g)}\right)-\frac{1}{\dg}\sum_i\ln\left(\frac{\Delta\Phi_i^{i+1}(g)}{\Delta\Phi_0^{\dg}(g)}\right)\\
     &= \ln\Delta\Phi_j^{j+1}(g)-\frac{1}{\dg}\sum_i\ln\Delta\Phi_i^{i+1}(g),
    \quad\text{with}\quad \Delta\Phi_{i}^j=\int_{g_i}^{g_{j}}\mathcal{N}(g,\sigma)(x)\,dx,\quad\text{and}\quad g_i=\frac{i}{\dg},
\end{split}
\end{equation}
so that its derivative can be compactly written as
\begin{equation}
    \hat{\mu}'_j(g) = h_j(g) - \frac{1}{\dg}\sum_i h_i(g),\quad\text{where}\quad h_i(g)=\frac{(\Delta\Phi_{i}^{i+1})'(g)}{\Delta\Phi_{i}^{i+1}(g)}.
    \label{eq:mujprime}
\end{equation}
Our goal is to bound $\hat{\mu}'_j(g)$, which can be quantified starting from considerations on the various $h_j(g)$. It can be proven that these are:
\begin{itemize}
\item equivalent modulo translations: $h_j(g)=h_{j-i}(g-i/\dg)$;
\item antisymmetric with respect to $g$ around the centerpoint $g_j^*=(g_{j+1}+g_j)/2$: $h_j(g_j^*+g)=-h_j(g_j^*-g)$;
\item antisymmetric with respect to $j$: $h_j(g_j^*+g)=-h_{\dg-j}(g_{\dg-j}^*-g)$;
\item decreasing: $h_j'(g)\leq0$;\todo{this I managed to prove formally, but it is a very long proof, and I don't think it makes sense to include it here}
\item convex for $g<g_j^*$: $g\lessgtr g_j^*\Longrightarrow h_j''(g)\gtrless 0$.\todo{This I checked by plotting the functions for various sigmas, and considering the behaviour at various limit expansions, but haven't formally proven it holds for the whole interval! - still, even as it is the proof is quite lengthy to be included here...}
\end{itemize}
We can gain a better intuition about how to effectively bound $\hat{\mu}'(g)$ by rewriting \eqref{eq:mujprime} using the equivalence under translations of $h_j(g)$:
\begin{equation}
    \hat{\mu}'_j(g) = \frac{1}{\dg}\sum_i\left(h_j(g)-h_j\left(g+\frac{j-i}{\dg}\right)\right)
\end{equation}
this shows that for each $j$ we are averaging the differences between $h_j(g)$ and the same function evaluated at $\dg$ equispaced points $g-(i-j)/\dg$. Since $h_j(g)$ is decreasing, we deduce that this difference is positive whenever $i<j$, and negative otherwise. We have then that the maximum absolute value of $\hat{\mu}_j'(g)$ is always attained for the most extreme $j$, since that guarantees that the largest number of terms share the same sign\todo{also, all terms referring to a $j$ are bound by the corresponding ones referring to $j+1$ due to concavity}. Without loss of generality (by symmetry), we can consider $j=\dg-1$, and we have 
\begin{equation}
    \max_j|\hat{\mu}'_j(g)| = \hat{\mu}'_{\dg-1}(g)\qquad\forall g\in[0,1].
\end{equation}
It suffices now to bound this quantity in $[0,1]$. Due to the concavity of $h_j(g)$, its maxima will be at the boundary, and specifically at $g=0$. This can be shown by simply comparing the values at $0$ and at $1$ (we drop the subscript $\dg-1$ and consider $h_{\dg-1}(g)=h(g)$ from now on):
\begin{equation}
\begin{split}
    \hat{\mu}'_{\dg-1}(0)-\hat{\mu}'_{\dg-1}(1) &=
    \frac{1}{\dg}\sum_{i=0}^{\dg-1}\left(h(0)-h\left(\frac{\dg-1-i}{\dg}\right)\right)
      - \frac{1}{\dg}\sum_{i=0}^{\dg-1}\left(h(1)-h\left(1+\frac{\dg-1-i}{\dg}\right)\right)\\
    & = \frac{1}{\dg}\sum_{i=0}^{\dg-1}\left(h(0)+h\left(\frac{\dg-1}{\dg}\right)-2h\left(\frac{i}{\dg}\right)\right)\\
    & = \frac{1}{\dg}\sum_{i=0}^{\dg-1}\left(h(0)+h\left(\frac{\dg-1}{\dg}\right)-\left(h\left(\frac{i}{\dg}\right)+h\left(\frac{\dg-1-i}{\dg}\right)\right)\right) \geq 0
\end{split}
\end{equation}
where we exploited the antisymmetry of $h_{\dg-1}(g)$ around $g_{\dg-1}^*=1-1/(2\dg)$ to aptly change the inner arguments, as well as the convexity of $h(g)$ for $g<g_{\dg-1}^*$ to state that $h(0)+h\left(\frac{\dg-1}{\dg}\right)\geq\left(h\left(\frac{i}{\dg}\right)+h\left(\frac{\dg-1-i}{\dg}\right)\right)$, for each $i$. This allows us to explicitly write the Lipschitz constant for the transformation $T_g(\vz)$ as
\begin{equation}
    \| T_g(\vz) - T_{\hat{g}}(\vz)\|_{\infty} \leq \hat{\mu}_{\dg-1}'(0)|g-\hat{g}|.
\end{equation}
\end{proof}

\end{section}

\begin{section}{Proof of Axioms for $T_g$ in \Cref{eq:t_g}}
\label{app:proof-axioms}
    Notice that $T_g$, thus defined, satisfies the group axioms at proper training ($L_\sG$ is minimized). In fact:
\begin{itemize}
    \item Neutral: $g=0 \;\;s.t.\;\; T_{0}(\vz) = \vz$. Easily proven since $ \mM_{g_0} = \widehat{\mM}_{g_0+0}$.
    \item Inverse: $g^{-1}=-g \;\;s.t.\;\; T_{g^{-1}} \circ T_{g} (\vz) = \vz$. Let $\vz^\prime = T_g(\vz)$, then $T_{g^{-1}}(\vz^\prime) = \vz^\prime - \mM_{g_0^\prime} + \widehat{\mM}_{g_0^\prime-g}$. Since $g_0^\prime = g_0 + g$, then $T_{g^{-1}} \circ T_{g} (\vz) = \vz - \mM_{g_0} + \widehat{\mM}_{g_0+g} - \mM_{g_0 + g} + \widehat{\mM}_{g_0 + g - g} = \vz$.
    \item Associativity: Similar reasoning as for the inverse property with 2 different elements.
    \item Closure: We work in $\R^D$ at representation level, so closure is verified.
\end{itemize}
\end{section}

\clearpage
\section{\method for Multiple Groups}
\label{app:multi_group}

Modern MSSL frameworks use complex augmentation stacks that compose several transformations. While learning structure with respect to a single group is interesting, one could benefit from learning such structure for a set of groups. For readability, we focus on the two group case ($\gG_A$ and $\gG_B$); but the following reasoning can easily be extended to more groups.

In order to model the interdependencies between groups and content, one can learn the joint distribution $P(\rc, \rg_A, \rg_B | \vx_k)$, where $\rg_A \in \sG_A$ and $\rg_B \in \sG_B$ are 2 random variables representing the respective group elements. Such approach implies that our backbone $f$ maps to $\R^{\dc\times |\sG_A| \times |\sG_B|}$. The marginal distributions are now obtained by summing over the non-desired dimensions (\eg over $\sC$ and $\sG_A$ to obtain $P(\rg_B|\vx_k)$). Using these new marginals, we define the multi-group loss as 
\begin{equation}
    \Loss_{\text{Multi-}\sG} = \frac{1}{2}\sum_{l=\{A, B\}} \Loss_{\sG_A} + \Loss_{\sG_B}.
\end{equation}
However, as the number of groups increases, modelling the joint distribution becomes intractable. In practice, keeping $\dc$ constant, the dimensionality of $\vz$ increases in $O(\dg^n)$ with the number of groups. 

To address scalability, we propose to relax the formulation and let our backbone $f$ map into $\R^{\dc\times (|\sG_A| + |\sG_B|)}$, so that the dimensionality of $\vz$ increases in $O(n\dg)$ with the number of groups. Using this relaxation we actually consider $\sG_A$ and $\sG_B$ independent, although their structure is learnt jointly during training. In practice, $\vz$ is divided into two blocks, with $|\sG_A|$ and $|\sG_B|$ columns each. In this scenario, $P(\rg_A|\vx_k)$ is obtained by summing over columns of the $\sG_A$ block, and $P(\rg_B|\vx_k)$ by summing over the columns of the $\sG_B$ block. The content marginal $P(\rc |\vx_k)$ is obtained by concatenating the sum over the rows of each group block.

\FloatBarrier
\section{Recovering the Transformation Parameter for a Von-Mises Distribution}
\label{app:recovervm}
Let $\vx_i$ be samples of a $\VM(\vx | \mu, \kappa)$ with unknown parameters $\mu$ and $\kappa$. We want to recover the parameter $\mu$, which corresponds to the group element that yields such $\VM$ prior. Let $\vr = \sum_i \vx_i$ be the baricenter of the samples with respect to the origin, then $\tilde{g} = \tilde{\mu} = \text{angle}(\vr)$.

\section{Empirical Equivariance: Additional Plots}
\label{app:empiricalequivariance}

\begin{figure*}[hbt!]
\centering
\begin{subfigure}{0.24\textwidth}
    \includegraphics[width=\textwidth]{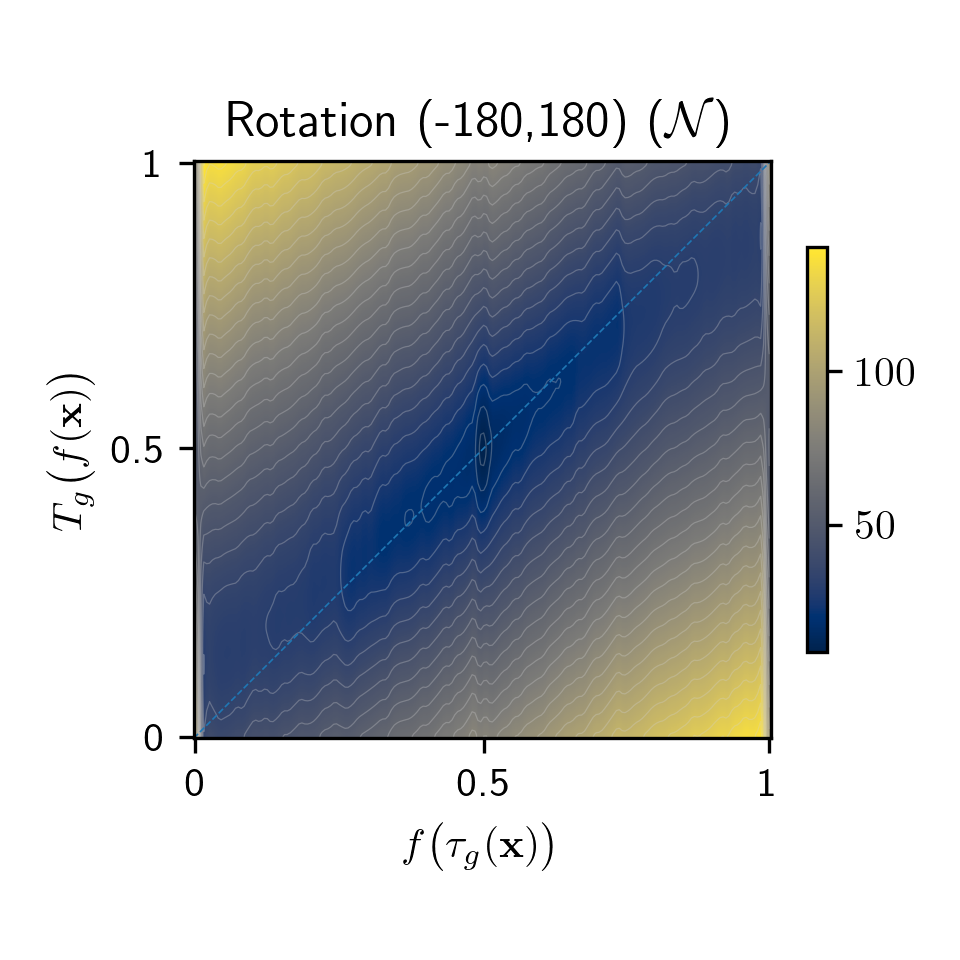}
    \label{fig:first}
\end{subfigure}
\hfill
\begin{subfigure}{0.24\textwidth}
    \includegraphics[width=\textwidth]{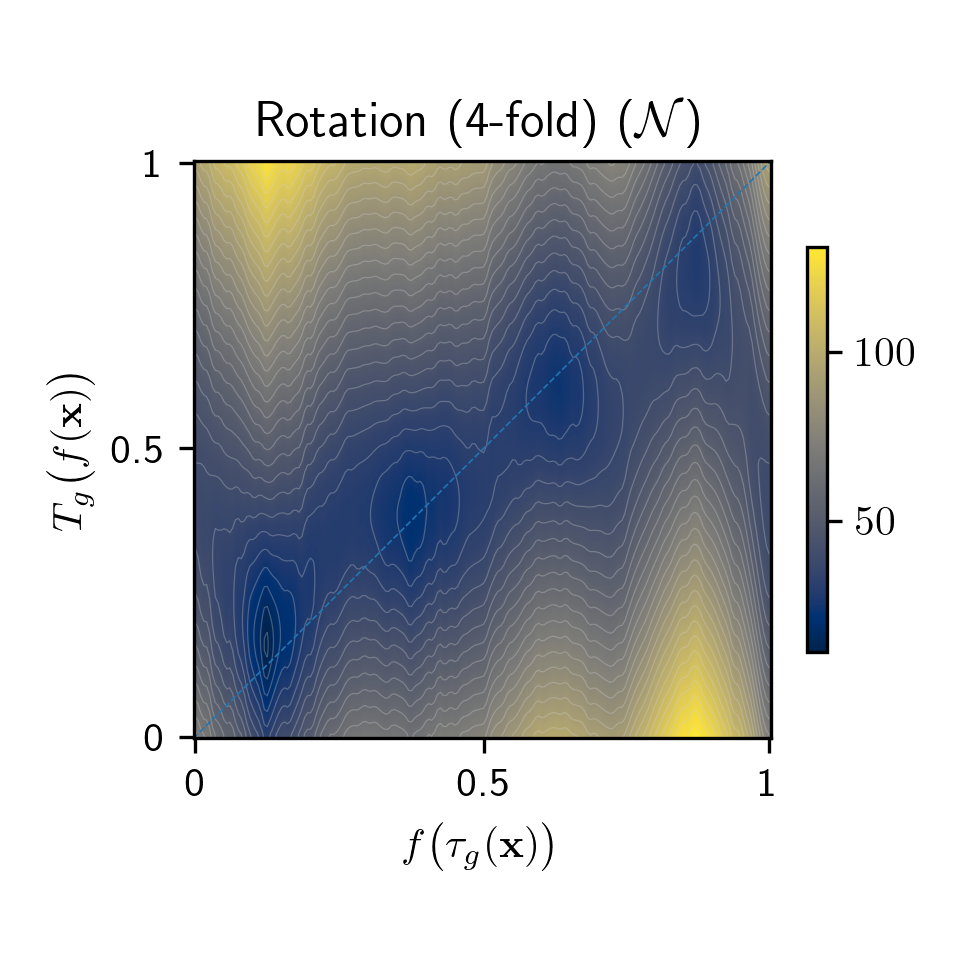}
    \label{fig:third}
\end{subfigure}
\hfill
\begin{subfigure}{0.24\textwidth}
    \includegraphics[width=\textwidth]{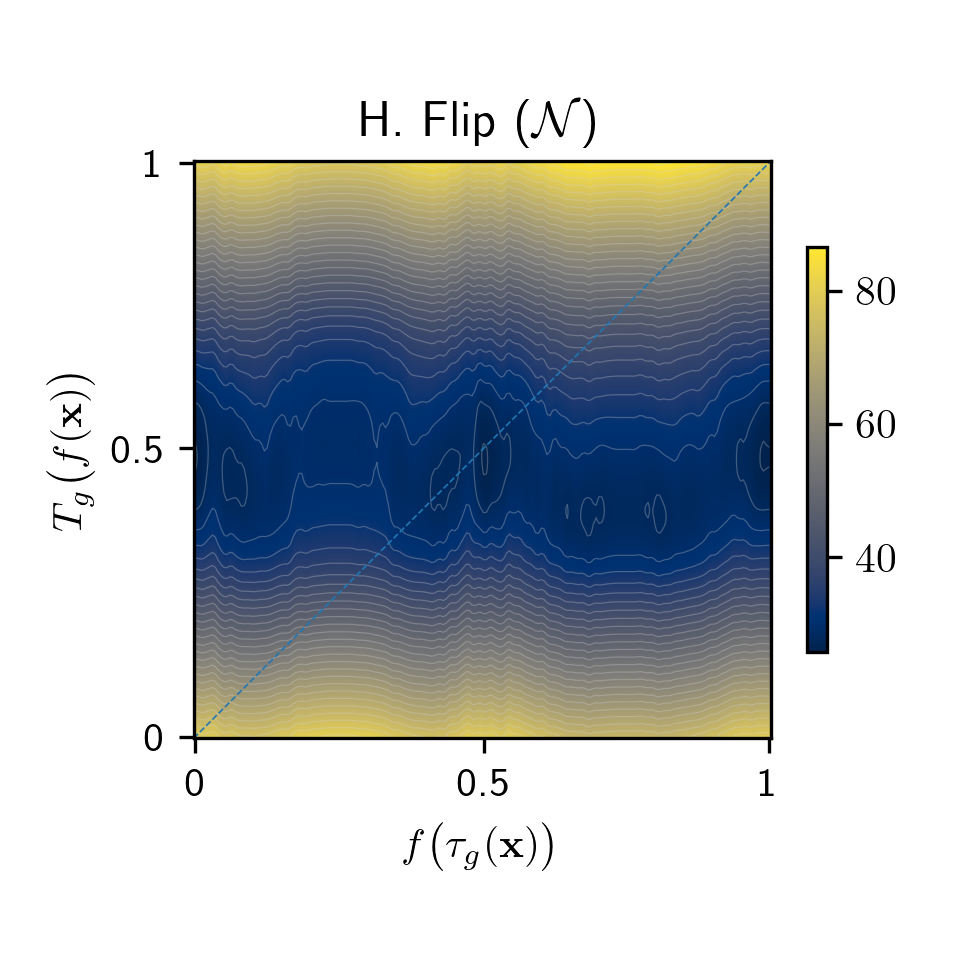}
    \label{fig:third}
\end{subfigure}
\hfill
\begin{subfigure}{0.24\textwidth}
    \includegraphics[width=\textwidth]{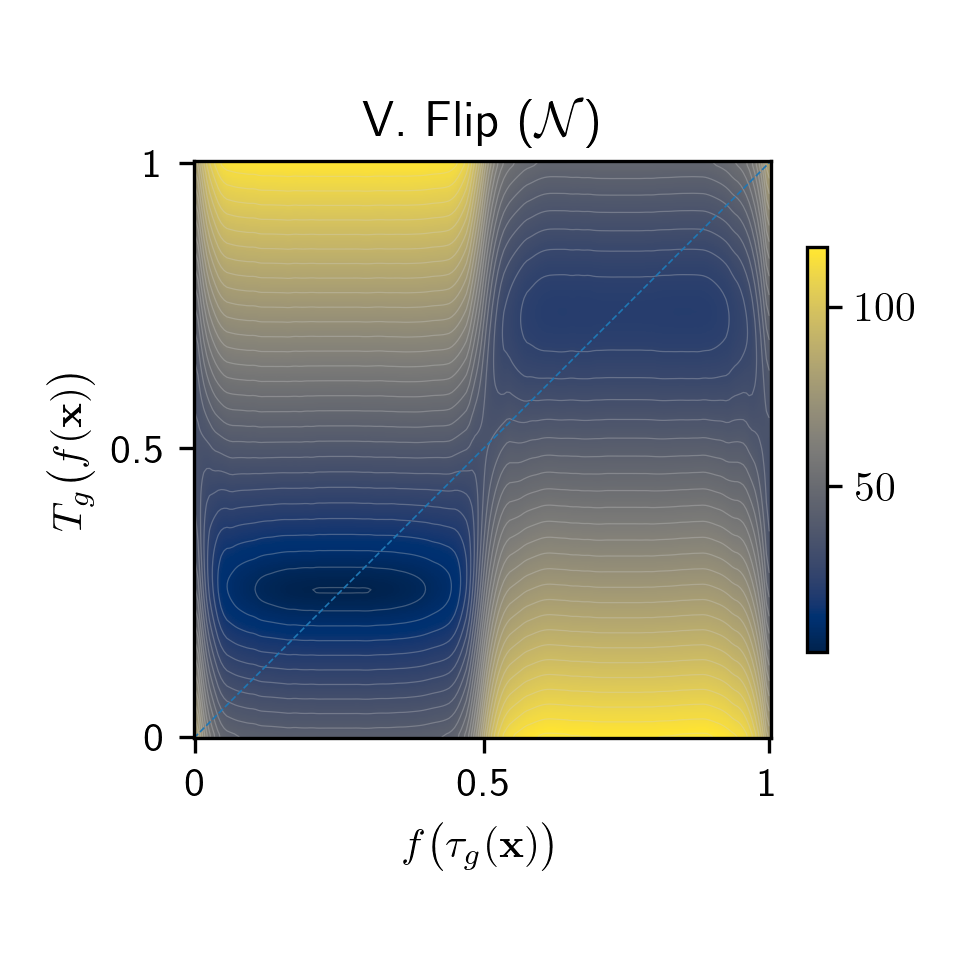}
    \label{fig:third}
\end{subfigure}
\vspace{-8mm}

\caption{Empirical validation of equivariance for cyclic groups with a non-cyclic Gaussian prior. Note the difference with the top row of \Figref{fig:empiricaleq}. In the Gaussian case, the cyclic nature of rotation and flip is not observed, and equivariance is less well satisfied.}
\label{fig:empirical_eq_gaussian}
\centering
\end{figure*}

\begin{figure}[htb!]
\caption{Examples of the transformations $\tau_g$ applied on the input images $\vx$ to obtain the plots in \Cref{fig:empiricaleq} and \Cref{fig:empirical_eq_gaussian}. For flips, we simulate a gradual flip by alpha-blending the 2 flipped images.}
\centering
\label{fig:image_tx}
\vskip 1mm
\includegraphics[width=0.7\columnwidth]{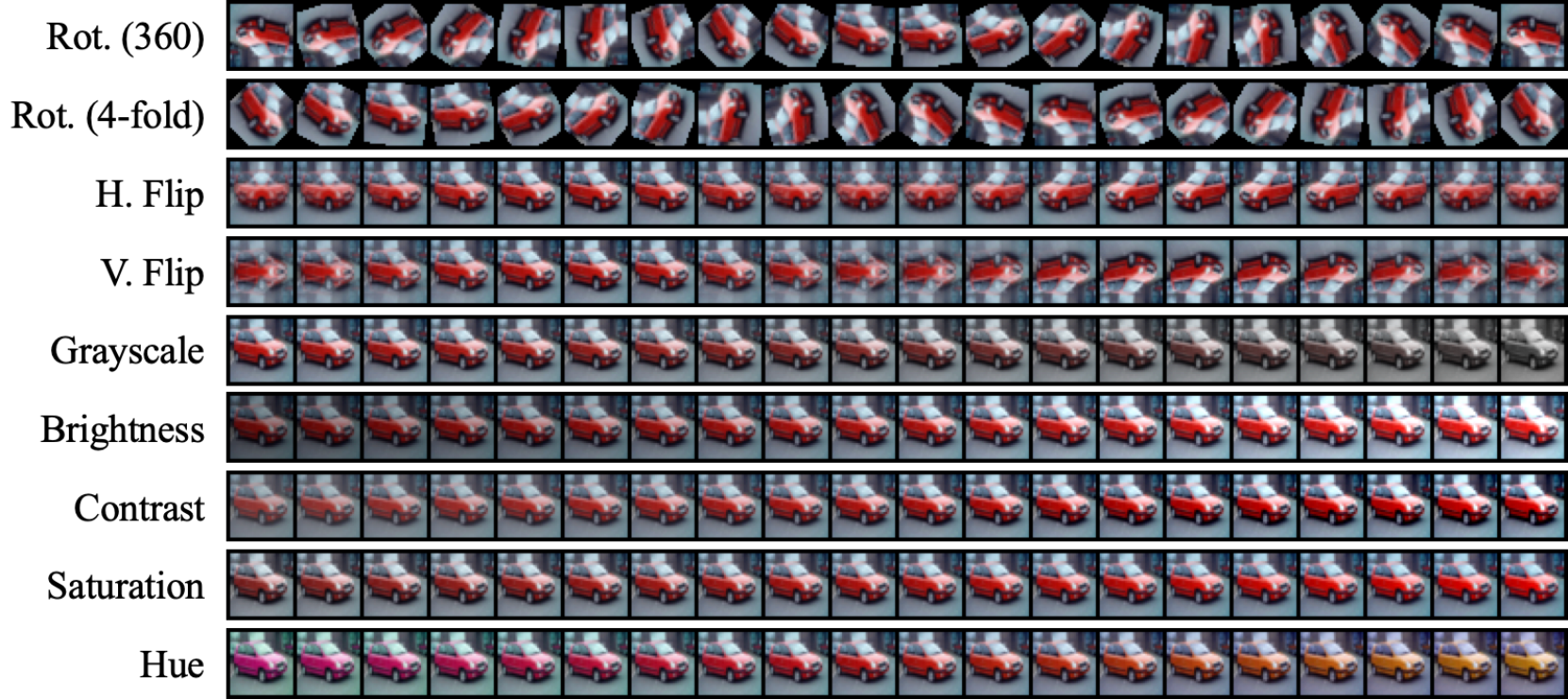}
\end{figure}

\FloatBarrier
\section{Reconstruction Error}
\label{app:rec_error}

In order to verify our hypothesis that structured representations are beneficial for generation, we measure the reconstruction error obtained with the decoders used in \Cref{sec:generation}. We use a mean squared error loss for reconstruction: $L_{rec}^{(i)} = \| d(f(\tau_g(\vx^{(i)}))) - \tau_g(\vx^{(i)}) \|_2^2$, where $d(\cdot)$ is a decoder network. In \Cref{fig:rec_loss} we plot the final test $L_{rec}$ on CIFAR-10 for decoders trained on frozen \method, ESSL and SimCLR representations, for some of the transformations analyzed.
The obtained reconstruction error with \method is up to 66\% smaller than with SimCLR (rotation (4-fold)) and up to 70\% smaller than with ESSL (grayscale).

\begin{figure}[htb!]
\caption{Reconstruction error (smaller is better) obtained with decoders trained on frozen \method, ESSL and SimCLR representations. The horizontal
dashed line shows the baseline error of SimCLR
with only RRC.}
\centering
\label{fig:rec_loss}
\vskip 1mm
\includegraphics[width=\columnwidth]{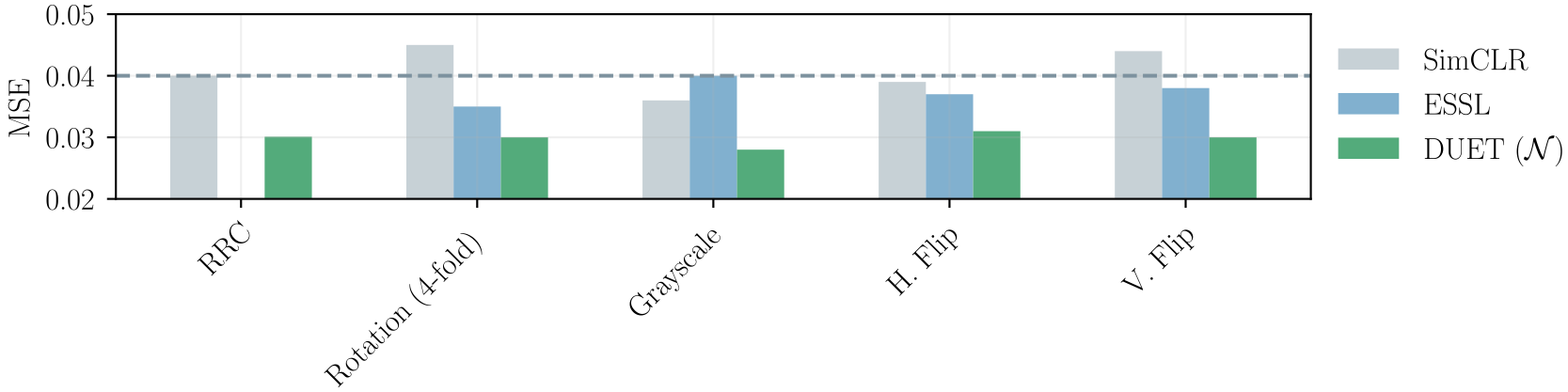}
\end{figure}

\FloatBarrier
\section{Full Stack Augmentations}
\label{app:fullstack-augs}

In Section~\ref{sec:fullstack} we report the performance of \method and other methods using the full SimCLR augmentation stack. More precisely, the augmentations used are:

\begin{itemize}
    \item \texttt{RandomResizedCrop(scale=(0.2, 1.0))}
    \item \texttt{ColorJitter(brightness=0.4, saturation=0.4, contrast=0.4, hue=0.1, p=0.8)}
    \item \texttt{RandomHorizontalFlip(p=0.5)}
    \item \texttt{RandomGrayscale(p=0.2)}
    \item \texttt{RandomGaussianBlur(kernel\_size=(3, 3), sigma=(0.1, 2.0), p=0.5)}
\end{itemize}

When we learn structure for groups that are not directly parameterized in this stack, we add a specific transformation. For example, for the \textit{vertical flip} group we add \texttt{RandomVerticalFlip(p=0.5)}. Or for rotations, we add a random rotation transformation in the stack.

\FloatBarrier
\section{Training Procedure}
\label{sec:train_procedure}
For all our experiments we use as backbone a ResNet-32 \cite{resnet} architecture with an input kernel of $3\times3$ and stride of 1. The output dimensionality of the ResNet is $\R^{512}$, which we reshape to $\R^{64\times 8}$ for a group granularity of $\dg=8$. Note that we do not add parameters, we only reshape the output of a vanilla ResNet to build our \method representations. Additional training parameters are shown in \Cref{tab:training_params}.

The detached decoders in \Cref{sec:generation} are also trained using the same procedure. The reconstructed images are RGB with $32 \times 32$ pixels. The decoder architecture is a ResNet-18 with swish activation functions, visual attention and GroupNorm \cite{groupnorm} normalization. 

\begin{table}[htb!]
    \caption{Training parameters.} 
    \label{tab:training_params}
    \centering
    \footnotesize
    \vskip 0.1in
    \begin{tabular}{ll}
    \toprule
    Batch size & 2048 \\
    Epochs & 800 \\
    Input images & RGB of $32 \times 32$ \\
    Learning rate & $0.0001$ \\
    Learning rate warm-up & 10 epochs \\
    Learning rate schedule & Cosine \\
    Optimizer & Adam($\beta=[0.9, 0.95]$) \\
    Weight decay & 0.0001 \\
    \bottomrule
\end{tabular}
\end{table}

\FloatBarrier

\subsection{Effect of $\lambda$, $\sigma$ and $\dg$}
\label{sec:ablation}
We perform a sweeping of $\lambda$ values between 0 and 1000. 
The first observation is that adding structure improves over SimCLR for all transformations (see \Cref{fig:sweep-lambda} in the Appendix).  However, color transformations and horizontal flips degrade performance if we strongly impose structure. This result hints that structure for such groups is harder to learn, or is less learnable from data (\eg the structure is ambiguous, as in the case of having flipped and non-flipped images in the dataset).
Interestingly, \stovez also improves slightly over SimCLR, showing that unsupervised structure is still helpful for the specific case of CIFAR-10. 
Overall, our results show that $\lambda=10$ is optimal for all transformations but \textit{scale}, \textit{rotations} and \textit{vertical flips} which can handle up to $\lambda=1000$.   

In \Cref{fig:sweep-sigma} we show the accuracy of SimCLR and \method at different $\sigma$ for all transformations analyzed. The violin plots show the median accuracy across transformations. We obtain an empirically optimal value of $\sigma=0.2$ for \method. Note that $\sigma=10$ is almost equivalent to a uniform target, thus not imposing any structure. In \Cref{fig:sigma-sweep-groups} we show the detailed results per transformation, observing that horizontal flip behaves better with a uniform target. Indeed, as observed in \Cref{sec:empiricalequivariance} and \Cref{sec:rep-cls}, with the datasets used horizontal flip is ambiguous and we cannot learn this symmetry from data.

Lastly, we found \method is quite insensitive to the choice of parameter $\dg$, based on results on CIFAR-10. We sweep $\dg=2, 4, 8, 16$ and the obtained accuracy changes by less than 1\%. We choose to use a reasonable value of $\dg=8$.

\begin{figure}[hbt!]
\caption{Sweep of \method's parameter $\sigma$. We find empirically that $\sigma\approx0.2$ works best. }
\centering
\label{fig:sweep-sigma}
\includegraphics[width=0.5\columnwidth]{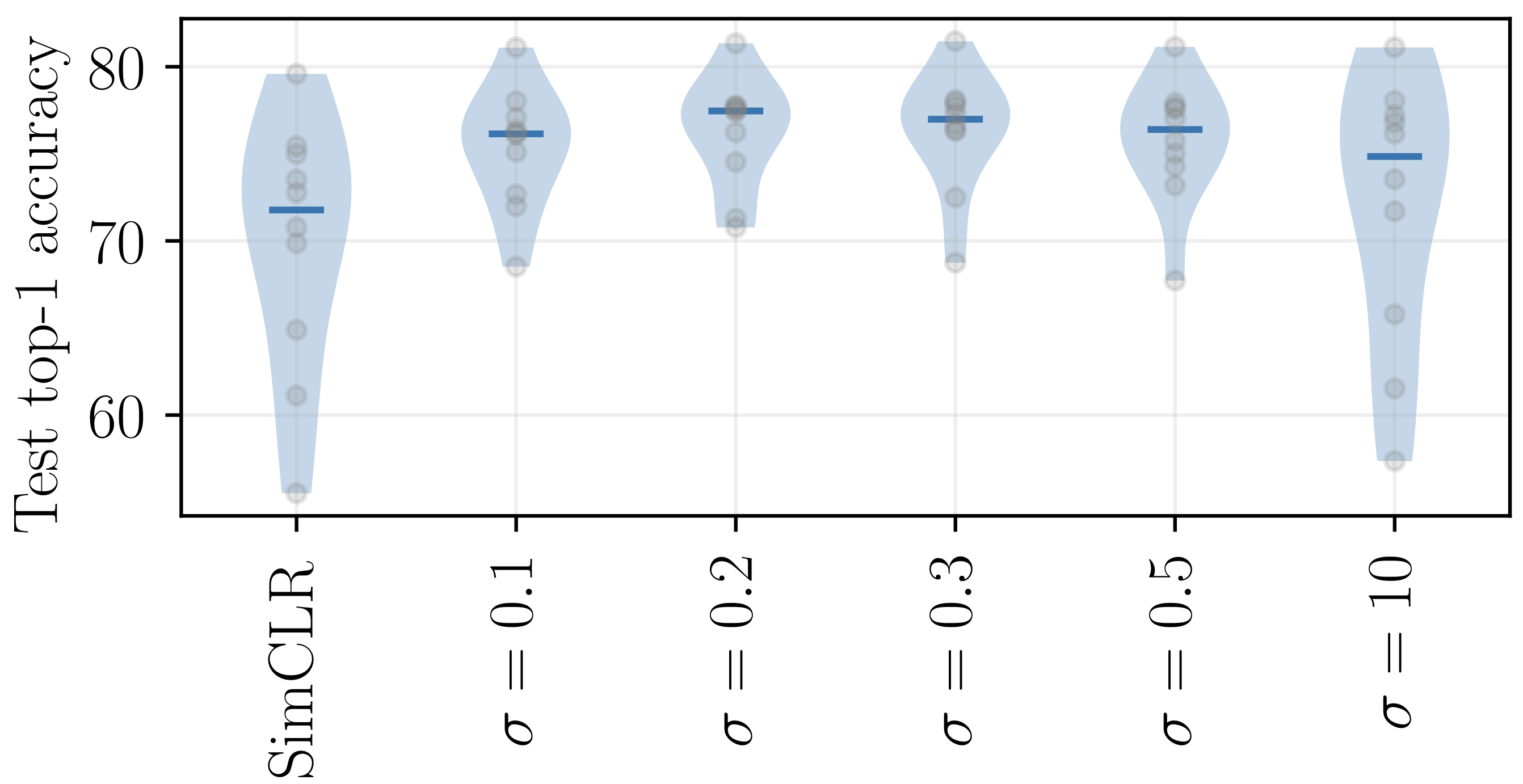}
\vskip -0.2in
\end{figure}

\begin{figure*}[hbt!]
\caption{Sweep of \method's parameter $\lambda$. Note that different transformations require different optimal $\lambda$s. }
\centering
\label{fig:sweep-lambda}
\includegraphics[width=0.5\columnwidth]{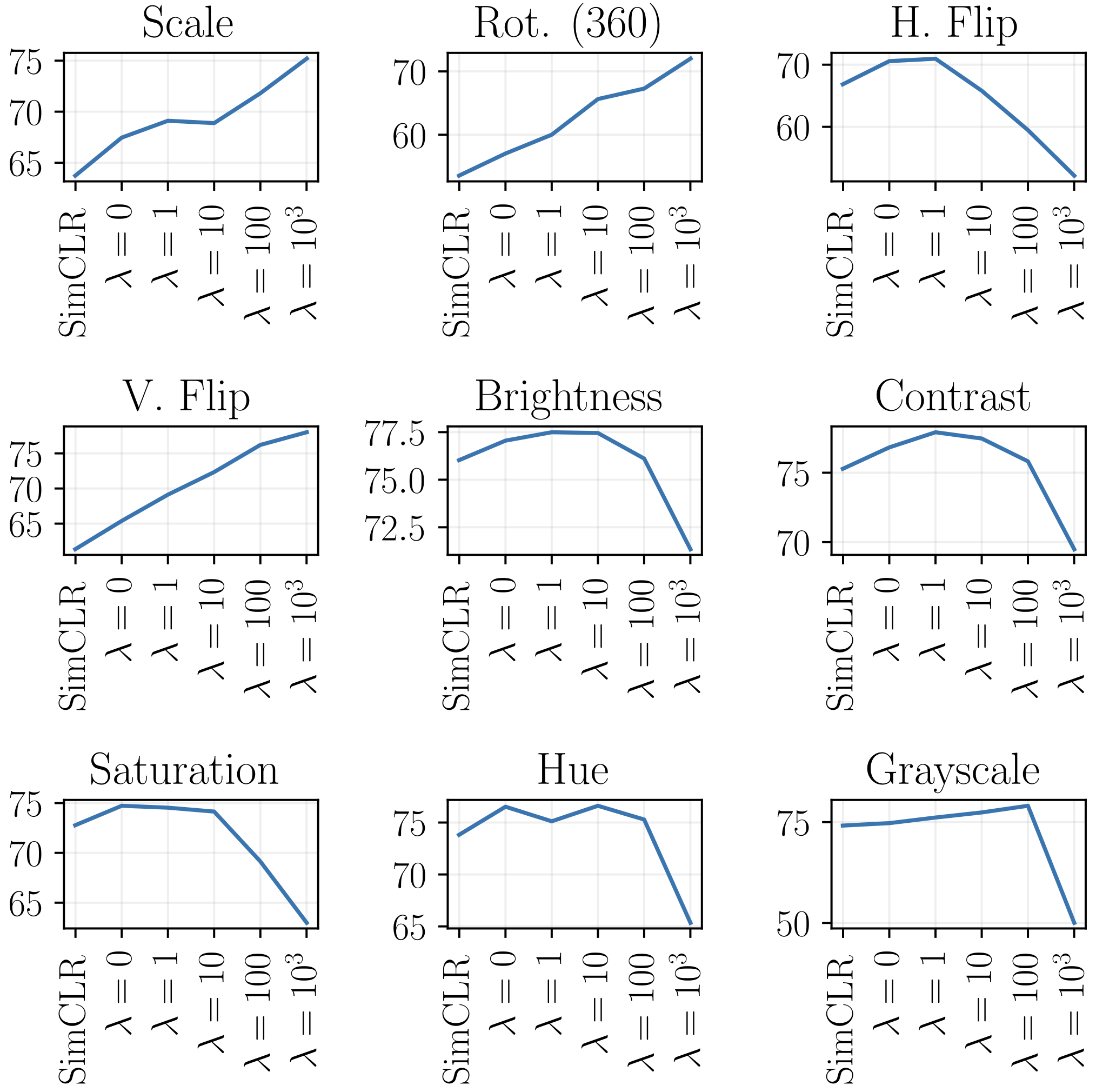}
\end{figure*}

\begin{figure*}[hbt!]
\caption{Test top-1 performance on CIFAR-10 as we modify the $\sigma$ parameter in \method. We report here the results per group.}
\centering
\label{fig:sigma-sweep-groups}
\includegraphics[width=0.75\columnwidth]{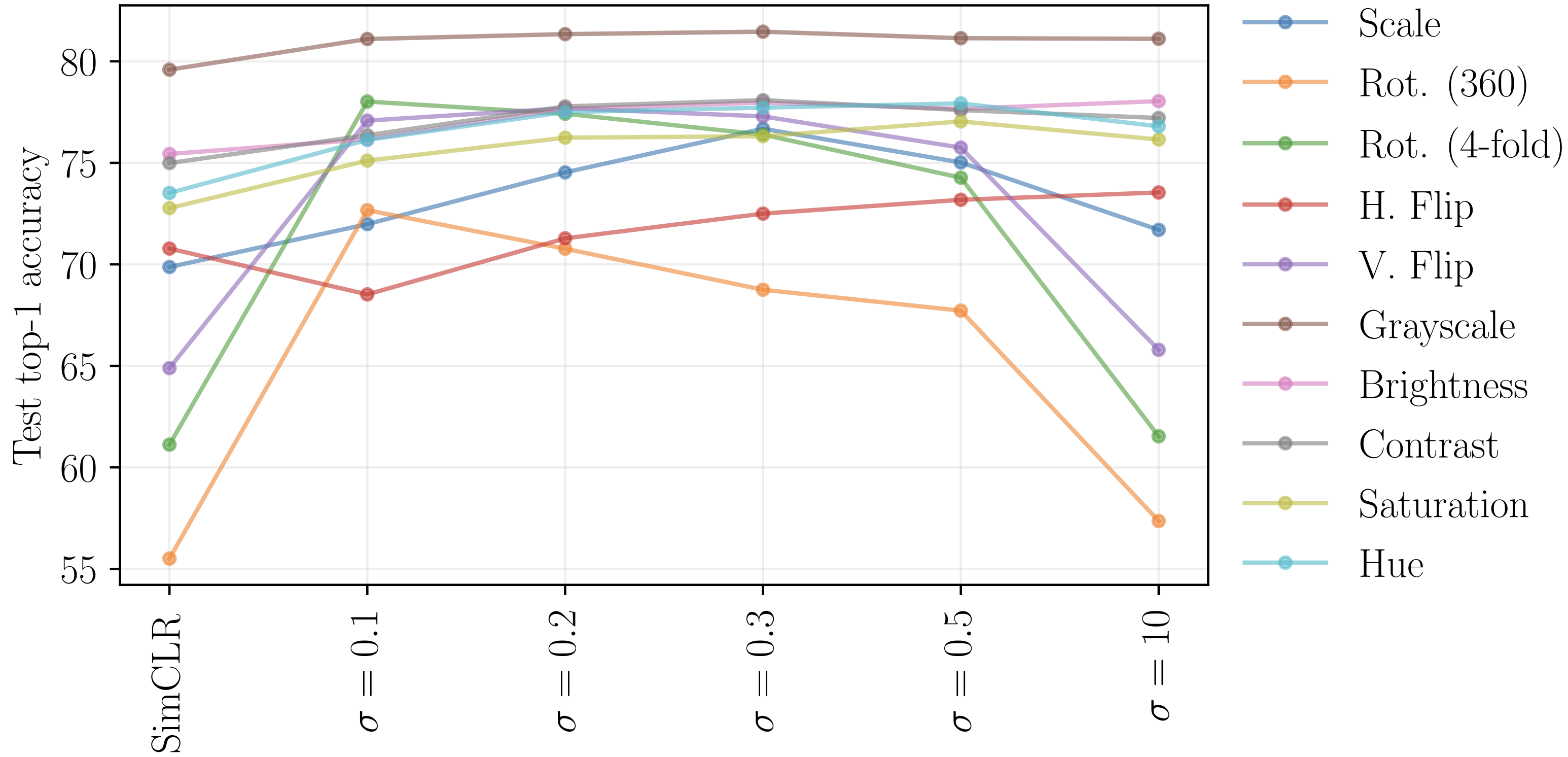}
\end{figure*}

\FloatBarrier
\section{Additional Results for Transfer Learning}
\label{app:transfer}

These results complement those summarized in \Cref{sec:transfer}.
We train a logistic regression classifier on the representations of the training split of each dataset. No augmentations are applied during the classifier training. At test time, we evaluate the classifier on the test set of each dataset.

In \Cref{fig:transfer} we report the difference in accuracies between \method and SimCLR.
\method's structure to rotations yields a gain of +21\% when transferring to Caltech101, and very important gains when transferring to other datasets like CIFAR-10, CIFAR-100, DTD or Pets. Structure to color transformations also proves beneficial, with a +6.36\% gain on Flowers (grayscale), 7.05\% on Food101 \cite{food101} (hue) and 7.13\% on CIFAR-100 (hue). Horizontal flip is the transformation that sees less gain, as expected given its ambiguity as shown in \Cref{fig:p_g_flips}.

It is interesting to see that \method achieves slightly worse performance for rotations or flips on the Flowers dataset. Indeed, this dataset contains many \textit{circular} flowers, which are rotation (or flip) invariant. In such situation, learning structure for rotations (or flips) should not give any gain. Actually, in \method we are trading off content for structure, so if the structure learnt is not useful, we are actually diminishing the expressivity of the final representations. 

Comparing with ESSL \Cref{fig:transfer-stove-essl}, \method achieves better transfer results for most of the datasets and transformations. Interestingly, ESSL improves over \method for geometric transformations on Flowers, due to the trade-off inherent in \method (see \Cref{sec:ambiguity}). For completeness, the results of ESSL compared to SimCLR are shown in \Cref{fig:transfer-essl-simclr}.

The linear regression for ESSL with grayscale on CIFAR-100 did not converge, thus we removed that result from the plots. 

\begin{figure*}[hbt!]
\caption{Difference in accuracies between \method and SimCLR, when transferring representations learnt on TinyImageNet to different datasets in the RRC+1 setting.}
\centering
\label{fig:transfer}
\includegraphics[width=0.75\columnwidth]{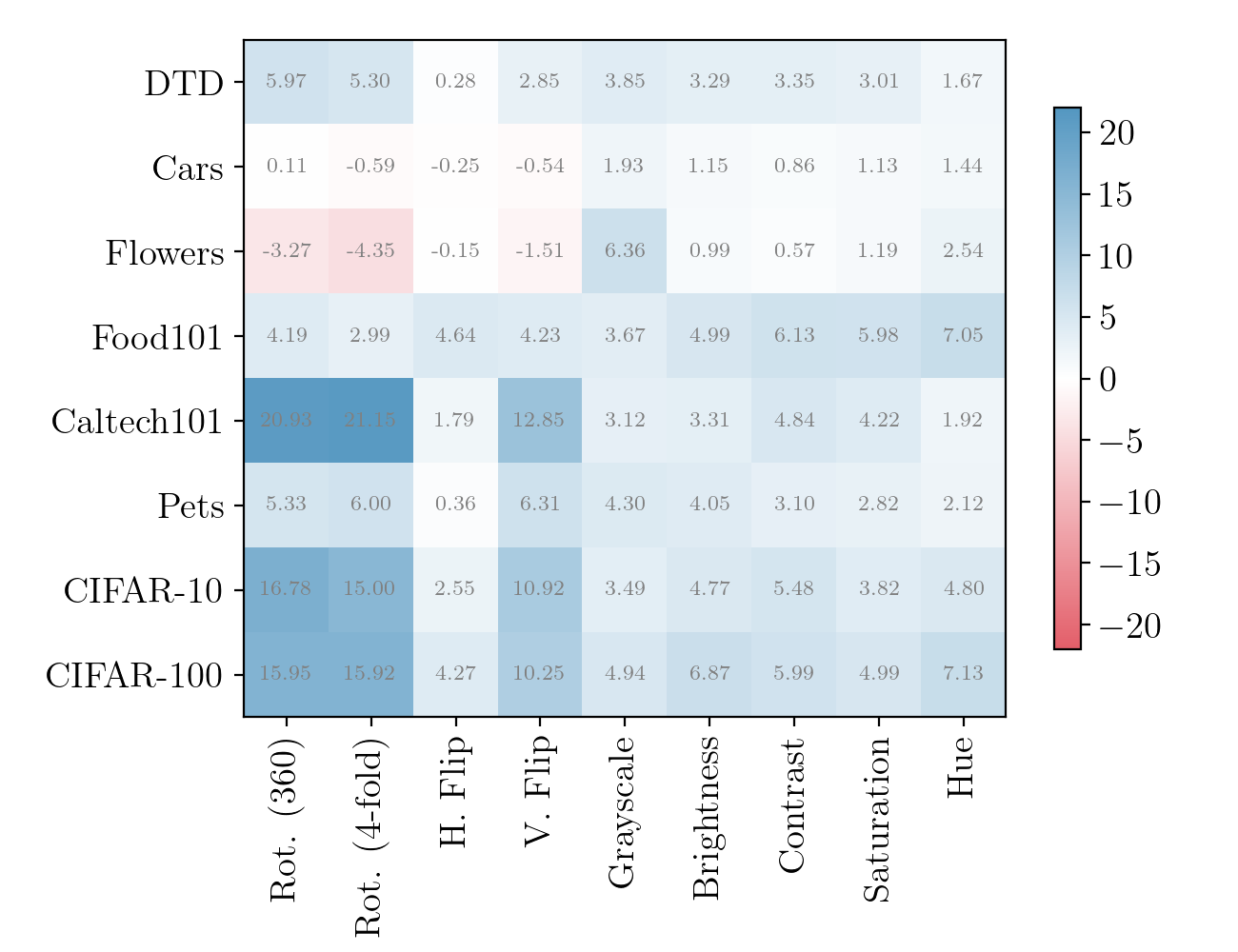}
\vskip -0.2in
\end{figure*}

\begin{figure*}[hbt!]
\caption{Difference in transfer accuracies between \method and ESSL, when transferring representations learnt on TinyImageNet to different datasets in the RRC+1 setting.}
\centering
\label{fig:transfer-stove-essl}
\includegraphics[width=0.75\columnwidth]{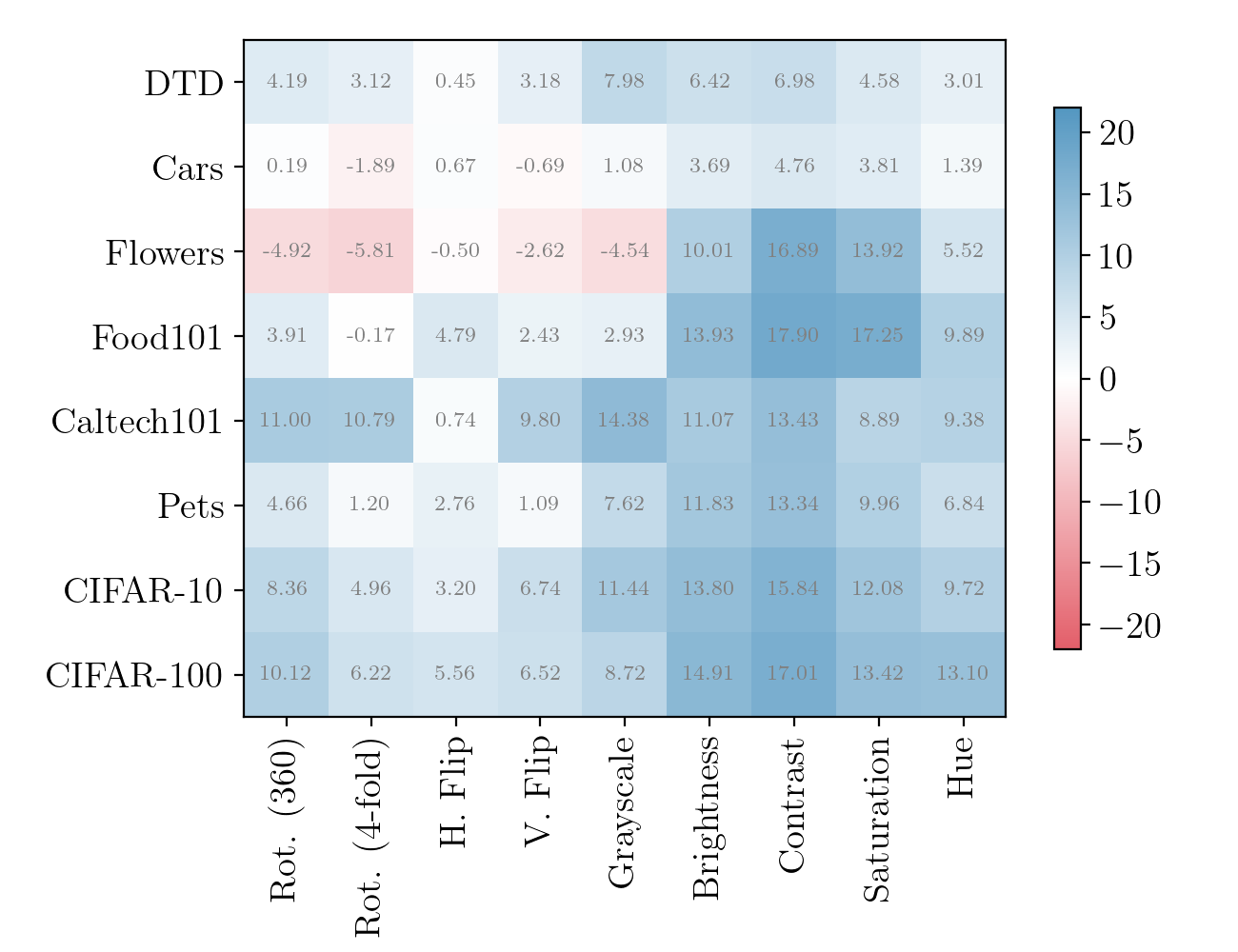}
\end{figure*}

\begin{figure*}[hbt!]
\caption{Difference in transfer accuracies between ESSL and SimCLR, when transferring representations learnt on TinyImageNet to different datasets in the RRC+1 setting.}
\centering
\label{fig:transfer-essl-simclr}
\includegraphics[width=0.75\columnwidth]{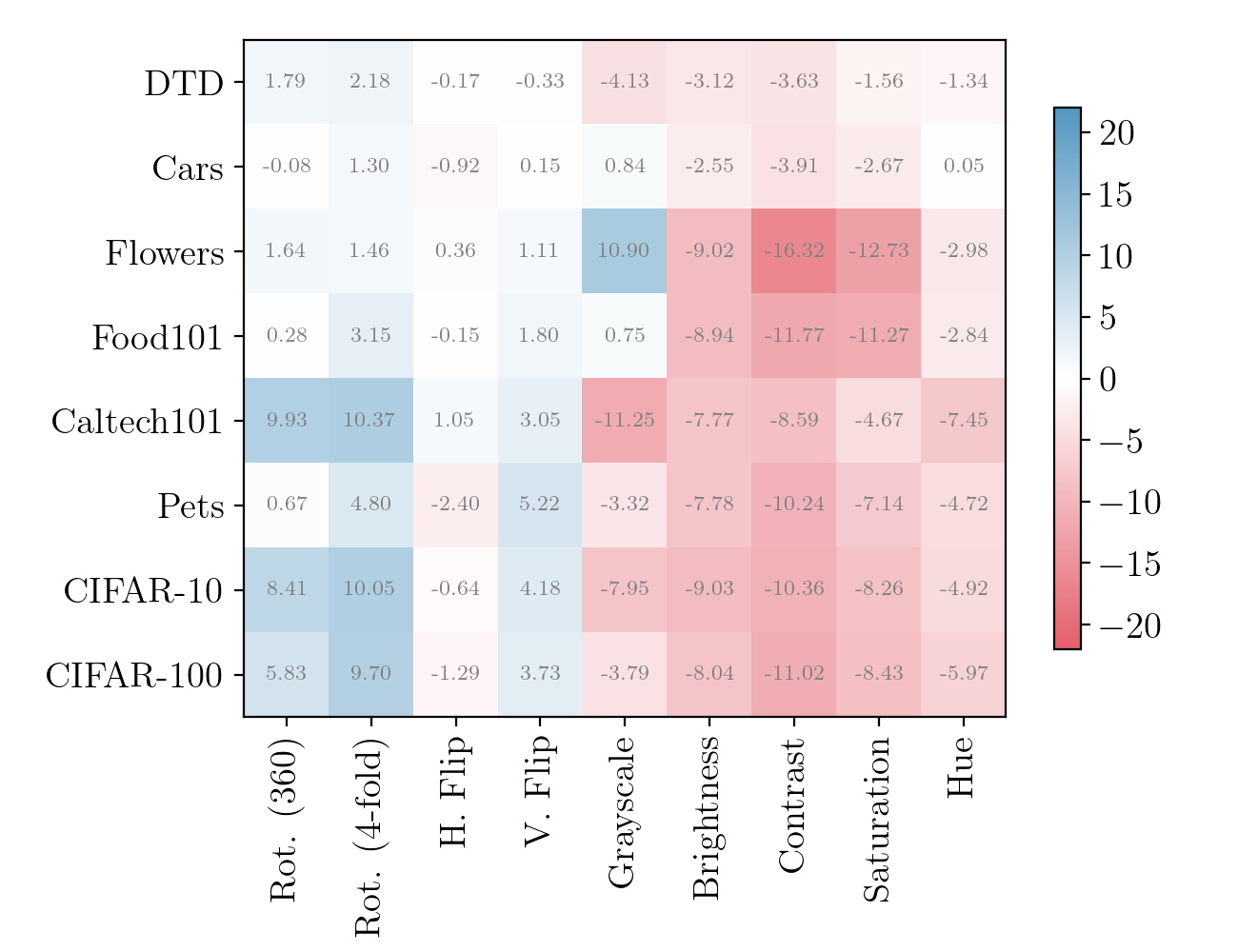}
\end{figure*}

\FloatBarrier
\section{Additional Results about Transformation Ambiguity}
\label{app:p_g}

\begin{figure}[htb!]
\caption{Observed $P(\rg | \vx)$ for different transformations, obtained from 100 randomly sampled CIFAR-10 images. Note the inherent ambiguity for color transformations, in addition to the one observed for horizontal flips in \Cref{fig:p_g_flips}.(left). Also, see how the modes of the distributions correspond to the mapped points in \Cref{tab:groups}.}
\centering
\label{fig:p_g_others}
\vskip 1mm
\begin{subfigure}[b]{0.6\columnwidth}
    \includegraphics[width=0.48\columnwidth]{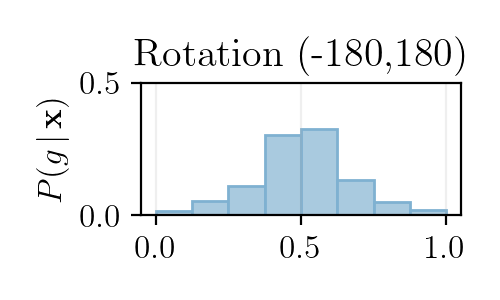}
    \includegraphics[width=0.48\columnwidth]{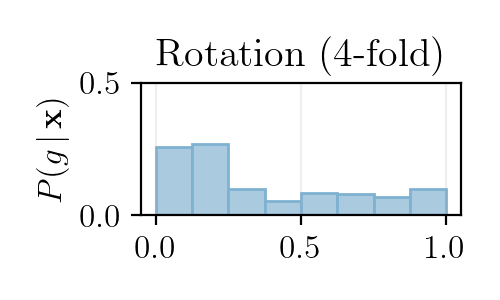}
    \includegraphics[width=0.48\columnwidth]{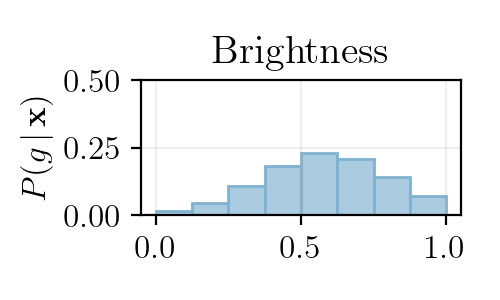}
    \includegraphics[width=0.48\columnwidth]{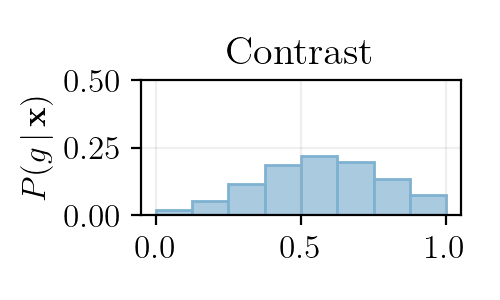}
    \includegraphics[width=0.48\columnwidth]{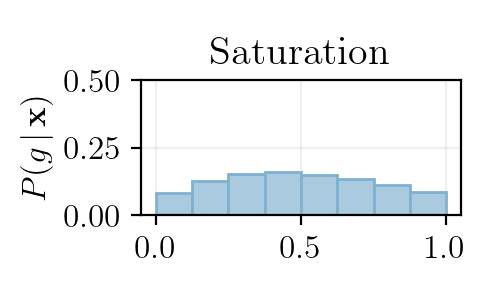}
    \includegraphics[width=0.48\columnwidth]{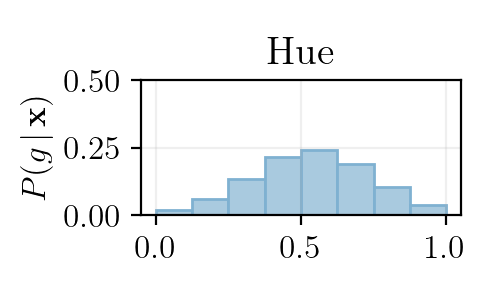}
    \includegraphics[width=0.48\columnwidth]{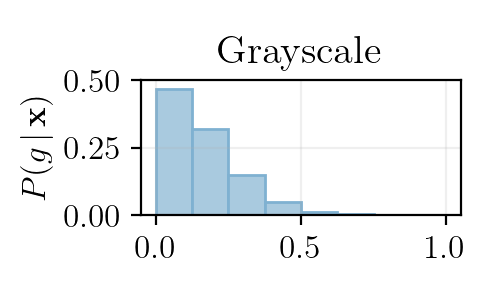}
\end{subfigure}
\vskip -0.2in
\end{figure}
\end{document}